\documentclass{article}

\usepackage[final,nonatbib]{nips_2018}

\usepackage[utf8]{inputenc} 
\usepackage[T1]{fontenc}    
\usepackage{hyperref}       
\usepackage{url}            
\usepackage{booktabs}       
\usepackage{amsfonts}       
\usepackage{nicefrac}       
\usepackage{microtype}      

\usepackage{amsmath,amsthm,amssymb} 
\usepackage{graphicx} 
\usepackage{listings} 
\usepackage{dsfont}
\usepackage{mathtools}
\usepackage{subcaption}

\usepackage{pgf}
\usepackage{tikz}
\usetikzlibrary{arrows,automata,positioning}

\def\G{{\mathcal G}}
\def\eqdef{\stackrel{\text{def}}{=}}

\DeclareMathOperator*{\argmax}{arg\,max}

\newcommand{\Ckapi}{C_{\kappa\rm{-API}}}

\newtheorem{proposition}{Proposition}

\newtheorem{lemma}[proposition]{Lemma}
\newtheorem{theorem}[proposition]{Theorem}

\newtheorem{defn}{Definition}
\usepackage{verbatim}
\newtheorem{remark}{Remark}

\usepackage{bbm}
\usepackage{algorithm,algorithmic}
\usepackage{varwidth}

\title{Multiple-Step Greedy Policies in Online and Approximate Reinforcement Learning}

\author{
  Yonathan Efroni\thanks{Department of Electrical Engineering, Technion, Israel Institute of Technology}\\
  \texttt{jonathan.efroni@gmail.com} \\
  \And
  Gal Dalal\footnotemark[1]\\
  \texttt{gald@campus.technion.ac.il} \\
  \And
  Bruno Scherrer\thanks{INRIA, Villers les Nancy, France}\\
  \texttt{bruno.scherrer@inria.fr} \\  
  \And
  Shie Mannor\footnotemark[1] \\
  \texttt{shie@ee.technion.ac.il} \\
}

\begin{document}

\maketitle

\begin{abstract}
  Multiple-step lookahead policies have demonstrated high empirical competence in Reinforcement Learning, via the use of Monte Carlo Tree Search or Model Predictive Control. In a recent work \cite{efroni2018beyond}, multiple-step greedy policies and their use in vanilla Policy Iteration algorithms were proposed and analyzed. In this work, we study multiple-step greedy algorithms in more practical setups. We begin by highlighting a counter-intuitive difficulty, arising with soft-policy updates: even in the absence of approximations, and contrary to the 1-step-greedy case, monotonic policy improvement is not guaranteed unless the update stepsize is sufficiently large. Taking particular care about this difficulty, we formulate and analyze online and approximate algorithms that use such a multi-step greedy operator. 
  
\end{abstract}

\section{Introduction}
The use of the 1-step policy improvement in Reinforcement Learning (RL)  was theoretically investigated under several frameworks, e.g., Policy Iteration (PI) \cite{puterman1994markov}, approximate PI \cite{bertsekas1995neuro, kakade:02, munos2003error}, and Actor-Critic \cite{konda1999actor}; its practical uses are abundant \cite{schulman2015trust,mnih2016asynchronous,silver2017mastering}. However, single-step based improvement is not necessarily the optimal choice. It was, in fact, empirically demonstrated that multiple-step greedy policies can perform conspicuously better. Notable examples arise from the integration of RL and Monte Carlo Tree Search \cite{browne2012survey,tesauro1997line,sheppard2002world,bouzy2004monte,silver2017mastering,silver2017mastering2} 
or Model Predictive Control  \cite{negenborn2005learning,ernst2009reinforcement,tamar2017learning}.

Recent work \cite{efroni2018beyond} provided guarantees on the performance of the multiple-step greedy policy and generalizations of it in PI. Here, we establish it in the two practical contexts of online and approximate PI. With this objective in mind, we begin by highlighting a specific difficulty: \emph{softly updating} a policy with respect to (w.r.t.) a multiple-step greedy policy does not necessarily result in improvement of the policy (Section \ref{sec:monotonicity_and_soft}). We find this property intriguing since monotonic improvement is guaranteed in the case of soft updates w.r.t. the 1-step greedy policy, and is central to the analysis of many RL algorithms \cite{konda1999actor,kakade:02,schulman2015trust}.
We thus engineer several algorithms to circumvent this difficulty and provide some non-trivial performance guarantees, that support the interest of using multi-step greedy operators.
These algorithms assume access to a generative model (Section~\ref{sec:online_kappa_PI}) or to an approximate multiple-step greedy policy (Section~\ref{sec:app_kappa_greedy}).




\section{Preliminaries} 
Our framework is the infinite-horizon discounted Markov Decision Process (MDP). An MDP is defined as the 5-tuple $(\mathcal{S}, \mathcal{A},P,R,\gamma)$ \cite{puterman1994markov}, where ${\mathcal S}$ is a finite state space, ${\mathcal A}$ is a finite action space, $P \equiv P(s'|s,a)$ is a transition kernel, $R \equiv r(s,a)$ is a reward function, and $\gamma\in(0,1)$ is a discount factor. Let $\pi: \mathcal{S}\rightarrow \mathcal{P}(\mathcal{A})$ be a stationary policy, where $\mathcal{P}(\mathcal{A})$ is a probability distribution on $\mathcal{A}$. Let $v^\pi \in \mathbb{R}^{|\mathcal{S}|}$ be the value of a policy $\pi,$ defined in state $s$ as $v^\pi(s) \equiv \mathbb{E}^\pi[\sum_{t=0}^\infty\gamma^tr(s_t,\pi(s_t))|s_0=s].$  For brevity, we respectively denote the reward and value at time $t$ by $r_t\equiv r(s_t,\pi_t(s_t))$ and $v_t\equiv v(s_t).$  It is known that
$
v^\pi=\sum_{t=0}^\infty \gamma^t (P^\pi)^t r^\pi=(I-\gamma P^\pi)^{-1}r^\pi,
$
with the component-wise values $[P^\pi]_{s,s'}  \triangleq P(s'\mid s, \pi(s))$ and $[r^\pi]_s \triangleq  r(s,\pi(s))$.
Lastly, 
let 
\begin{equation}
\label{eq: q pi def}
q^{\pi}(s,a) = \mathbb{E}^\pi[\sum_{t=0}^\infty\gamma^tr(s_t,\pi(s_t))\mid s_0=s,a_0=a].
\end{equation}

Our goal is to find a policy $\pi^*$ yielding the optimal value $v^*$ such that
\begin{align}
v^* = \max_\pi (I-\gamma P^\pi)^{-1} r^\pi = (I-\gamma P^{\pi^*})^{-1} r^{\pi^*}. \label{mdp}
\end{align}
This goal can be achieved using the three classical operators (equalities hold component-wise): 
\begin{align*}
\forall v,\pi,~  T^\pi v & =  r^\pi +\gamma P^\pi v, \\
\forall v,~  T v & =  \max_\pi T^\pi v, \\
\forall v,~\G(v)&= \{\pi : T^\pi v = T v\}, 
\end{align*}
where $T^\pi$ is a linear operator, $T$ is the optimal Bellman operator and both $T^\pi$ and $T$ are $\gamma$-contraction mappings w.r.t. the max norm. It is known that the unique fixed points of $T^\pi$ and $T$ are $v^\pi$ and $v^*$, respectively. The set $\G(v)$ is the standard set of 1-step greedy policies w.r.t. $v$. 

\section{The $h$- and $\kappa$-Greedy Policies} \label{sec:dp_multiple_step}
In this section, we bring forward necessary definitions and results on two classes of multiple-step greedy policies: $h$- and $\kappa$-greedy \cite{efroni2018beyond}.
Let $h\in \mathbb{N} \backslash \{0\}$. The $h$-greedy policy $\pi_h$ outputs the first optimal action out of the sequence of actions solving a non-stationary, $h$-horizon control problem as follows:
\begin{align*}
\forall s\in \mathcal{S},\ \pi_h(s) \in \arg\max\limits_{\pi_0} \max\limits_{\pi_1,..,\pi_{h-1}} \mathbb{E}^{\pi_0\dots\pi_{h-1}}\left[\sum_{t=0}^{h-1}\gamma^t r(s_t,\pi_t(s_t))+\gamma^h v(s_h)\mid s_0=s\right].
\end{align*}
Since the $h$-greedy policy can be represented as the 1-step greedy policy w.r.t. $T^{h-1}v$, the set of $h$-greedy policies w.r.t. $v$, $\G_h(v)$, can be formally defined as follows:
\begin{align*}
\forall v,\pi,~ T_h^\pi v &= T^{\pi}T^{h-1} v,  \\
\forall v,~\G_h(v)&= \{\pi: T_h^\pi v = T^h v\}. 
\end{align*}

Let $\kappa\in[0,1]$. The set of $\kappa$-greedy policies w.r.t. a value function $v$, $\G_\kappa(v)$, is defined using the following operators:
\begin{align}
\forall v,\pi,~ T_\kappa^\pi v & =  (I-\kappa\gamma P^\pi)^{-1}(r^\pi+(1-\kappa)\gamma P^\pi v) \nonumber \\
\forall v,\ T_\kappa v &= \max_{\pi} T_\kappa^{\pi} v = \max_{\pi}  (I-\kappa\gamma P^{\pi})^{-1}  (r^{\pi}+(1-\kappa)\gamma P^{\pi} v)\label{kappa_mdp}\\
\forall v,~\G_\kappa(v) &= \{ \pi: T_\kappa^\pi v = T_\kappa v\} \nonumber.
\end{align}

\begin{remark}
	\label{rem: applying T kappa}
	A comparison of \eqref{mdp} and \eqref{kappa_mdp} reveals that finding the $\kappa$-greedy policy is equivalent to solving a $\kappa\gamma$-discounted MDP with shaped reward $r^\pi_{v,\kappa}\eqdef r^{\pi}+(1-\kappa)\gamma P^{\pi} v$.
\end{remark}
In \cite[Proposition~11]{efroni2018beyond}, the $\kappa$-greedy policy was explained to be interpolating over all geometrically $\kappa$-weighted $h$-greedy policies. It was also shown that for $\kappa=0,$ the 1-step greedy policy is restored, while for $\kappa=1,$ the $\kappa$-greedy policy is the optimal policy.

Both $T_\kappa^\pi$ and $T_\kappa$ are $\xi_\kappa$ contraction mappings, where $\xi_\kappa=\frac{\gamma(1-\kappa)}{1-\gamma\kappa}\in [0,\gamma]$. Their respective fixed points are $v^\pi$ and $v^*$. For brevity, where there is no risk of confusion, we shall denote $\xi_\kappa$ by $\xi.$ Moreover, in \cite{efroni2018beyond} it was shown that both the $h$- and $\kappa$-greedy policies w.r.t. $v^\pi$ are strictly better then $\pi$, unless $\pi = \pi^*$.

Next, let 
\begin{equation}
\label{eq: q pi kappa def}
q^\pi_\kappa(s,a)=\max_{\pi'}\mathbb{E}^{\pi'}[\sum_{t=0}^\infty(\kappa\gamma)^t(r(s_t,\pi'(s_t)) + \gamma (1-\kappa) v^\pi(s_{t+1})\mid s_0=s,a_0=a].
\end{equation}
The latter is the optimal $q$-function of the surrogate, $\gamma\kappa$-discounted MDP with $v^\pi$-shaped reward (see Remark \ref{rem: applying T kappa}). Thus, we can obtain a $\kappa$-greedy policy, $\pi_\kappa\in \G_\kappa(v^\pi)$, directly from $q^\pi_\kappa:$ 
\begin{align*}
\pi_\kappa(s)\in \arg\max_a q^\pi_\kappa(s,a),\ \forall s\in\mathcal{S}.
\end{align*}
See that the greedy policy w.r.t. $q^{\pi}_{\kappa=0}(s,a)$ is the 1-step greedy policy since ${q^{\pi}_{\kappa=0}(s,a)\!=\!q^{\pi}(s,a).}$


\section{Multi-step Policy Improvement and Soft Updates}\label{sec:monotonicity_and_soft}
In this section, we focus on policy improvement of multiple-step greedy policies, performed with soft updates. Soft updates of the 1-step greedy policy have proved necessary and beneficial in prominent algorithms \cite{konda1999actor,kakade:02,schulman2015trust}. Here, we begin by describing an intrinsic difficulty in selecting the step-size parameter $\alpha \in (0,1]$ when updating with multiple-step greedy policies. Specifically, denote by $\pi'$ such multiple-step greedy policy w.r.t. $v^\pi.$ Then, $\pi_{\text{new}}= (1-\alpha)\pi+\alpha\pi'$ is not necessarily better than $\pi$.

\begin{theorem}\label{thm:mixture_policies_and_monotonic_improvement} For any MDP, 
	let $\pi$ be a policy and $v^\pi$ its value. Let $\pi_\kappa \in\G_\kappa(v^{\pi})$ and $\pi_h \in\G_h(v^{\pi})$  with $\kappa \in [0,1]$ and $h>1$. Consider the mixture policies with ${\alpha\in(0,1]},$
	\begin{align*}
	&\pi(\alpha,\kappa) \eqdef (1-\alpha)\pi+\alpha \pi_\kappa,\\
	&\pi(\alpha,h) \eqdef (1-\alpha)\pi+\alpha \pi_h.
	\end{align*}
	Then we have the following equivalences:
	\begin{enumerate}
		\item{The inequality $v^{\pi(\alpha,\kappa)}\geq v^{\pi}$ holds for all MDPs if and only if  $\alpha \in [\kappa,1]$.}
		\item{The inequality $v^{\pi(\alpha,h)}\geq v^{\pi}$ holds for all MDPs  if and only if  $\alpha=1$.} 
	\end{enumerate}
	The above inequalities hold entry-wise, with strict inequality in at least one entry unless $v^\pi = v^*$.
\end{theorem}
\emph{Proof sketch. } See Appendix~\ref{supp:mixture_policies_and_monotonic_improvement} for the full proof. Here, we only provide a counterexample demonstrating the potential non-monotonicity of $\pi(\alpha,\kappa)$ when the stepsize $\alpha$ is not big enough. One can show the same for $\pi(\alpha,h)$ with the same example.

\begin{figure}
	\centering
	\resizebox{2.1in}{!}{
		\begin{tikzpicture}[->,>=stealth',shorten >=1pt,auto,node distance=2.5cm,
		semithick, state/.style={circle, draw, minimum size=1.0cm}]
		\tikzstyle{every state}=[thick]
		]
		
		\node[state] (S0) {$s_0$};
		\node[state] (S1) [right of=S0] {\large $s_1$};
		\node[state] (S2) [right of=S1] {\large $s_2$};
		\node[state] (S3) [above of=S1] {\large $s_3$};
		
		\path (S0) edge  [loop above] node[pos=0.05,left]{\Large $a_0$} node {\Large $0$} (S0)
		edge     node[pos=0.15,below ]{\Large $a_1$}         node {\Large $0$} (S1)
		(S1) edge     node[pos=0.1,left]{\Large $a_0$} node {\Large $0$} (S3)
		edge     node[pos=0.15,below ]{\Large $a_1$}         node {\Large $0$} (S2)
		(S2) edge [loop above] node[pos=0.05,left]{\Large $a_0$} node {\Large $1$} (S2)
		(S3) edge  [loop above] node[pos=0.05,left]{\Large $a_0$} node {\Large $-c$} (S3);
		
		\end{tikzpicture}
	}
	\caption{The Tightrope Walking MDP used in the counter example of Theorem \ref{thm:mixture_policies_and_monotonic_improvement}.}
	\label{fig:thm_walking_a_rope_MDP}
\end{figure}
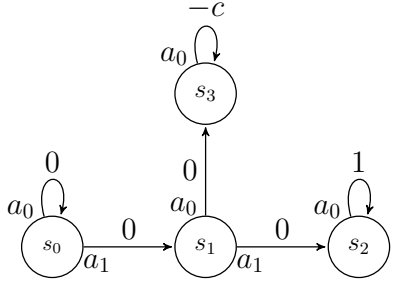

Consider the Tightrope Walking MDP in Fig.~\ref{fig:thm_walking_a_rope_MDP}. It describes the act of walking on a rope:  in the initial state $s_0$ the agent approaches the rope, in $s_1$ the walking attempt occurs, $s_2$ is the goal state and $s_3$ is repeatedly met if the agent falls from the rope, resulting in negative reward.

First, notice that by definition, $\forall v,~\pi^* \in \G_{\kappa=1}(v).$ We call this policy the ``confident'' policy.
Obviously, for any discount factor $\gamma \in (0,1)$, $\pi^*(s_0)=a_1$ and $\pi^*(s_1)=a_1.$ Instead, consider the ``hesitant'' policy $\pi_0(s) \equiv a_0 ~ \forall s$. We now claim that for any $\alpha\in(0,1)$ and 
\begin{equation}
\label{eq: c lower bound}
c>\frac{\alpha}{1-\alpha}
\end{equation}
the mixture policy, $\pi(\alpha,\kappa=1)=(1-\alpha)\pi_0+\alpha\pi^*$, is not strictly better than $\pi_0.$ To see this, notice that $v^{\pi_0}(s_1) < 0$ and $v^{\pi_0}(s_0)=0;$ i.e., the agent accumulates zero reward if she does not climb the rope. Thus, while $v^{\pi_0}(s_0)=0,$ taking any mixture of the confident and hesitant policies can result in $v^{\pi(\alpha,\kappa=1)}(s_0)<0,$ due to the portion of the transition to $s_1$ and its negative contribution. Based on this construction, let $\kappa\in[0,1].$ To ensure $\pi^*\in\G_\kappa(v^\pi),$ we find it is necessary that 
\begin{equation}
\label{eq: c upper bound}
c\leq\frac{\kappa}{1-\kappa}.
\end{equation}

To conclude, if both \eqref{eq: c lower bound} and \eqref{eq: c upper bound} are satisfied, the mixture policy does not improve over $\pi_0$. Due to the monotonicity of $\frac{x}{1-x},$ such a choice of $c$ is indeed possible for $\alpha<\kappa$.  \qedsymbol

Theorem~\ref{thm:mixture_policies_and_monotonic_improvement} guarantees monotonic improvement for the 1-step greedy policy as a special case when $\kappa=0$. Hence, we get that for any $\alpha\in(0,1],$ the mixture of any policy $\pi$ and the 1-step greedy policy w.r.t. $v^\pi$ is monotonically better then $\pi$. To the best of our knowledge, this result was not explicitly stated anywhere. Instead, it appeared within proofs of several famous results, e.g, \cite[Lemma~5.4]{konda1999actor}, \cite[Corollary~4.2]{kakade:02}, and \cite[Theorem~1]{scherrer2014local}.

In the rest of the paper, we shall focus on the $\kappa$-greedy policy and extend it to the online and the approximate cases. The discovery that the $\kappa$-greedy policy w.r.t. $v^\pi$ is not necessarily strictly better than $\pi$ will guide us in appropriately devising algorithms.


\section{Online $\kappa$-Policy Iteration with Cautious Soft Updates}\label{sec:online_kappa_PI}

In \cite{efroni2018beyond}, it was shown that using the $\kappa$-greedy policy in the improvement stage leads to a convergent PI procedure -- the $\kappa$-PI algorithm. This algorithm repeats i) finding the optimal policy of small-horizon surrogate MDP with shaped reward, and ii) calculating the value of the optimal policy and use it to shape the reward of next iteration. Here, we devise a practical version of $\kappa$-PI, which is model-free, online and runs in two timescales; i.e, it performs i) and ii) simultaneously.


The method is depicted in Algorithm~\ref{alg:async_kappaPI}. It is similar to the asynchronous PI analyzed in \cite{perkins2013asynchronous}, except for two major differences. First, the fast timescale tracks both $q^\pi,q^{\pi}_\kappa$ and not just $q^{\pi}$. Thus, it enables access to \emph{both} the 1-step-greedy and $\kappa$-greedy policies. The 1-step greedy policy is attained via the $q^\pi$ estimate, which is plugged into a $q$-learning \cite{watkins1992q} update rule for obtaining the $\kappa$-greedy policy. The latter essentially solves the surrogate $\kappa\gamma$-discounted MDP (see Remark \ref{rem: applying T kappa}). The second difference is in the slow timescale, in which the policy is updated using a new operator, $b_s$, as defined below. To better understand this operator, first notice that in Stochastic Approximation methods such as Algorithm~\ref{alg:async_kappaPI}, the policy is improved using soft updates with decaying stepsizes. However, as Theorem \ref{thm:mixture_policies_and_monotonic_improvement} states, monotonic improvement is not guaranteed below a certain stepsize value. Hence, for $q,q_\kappa \in \mathbb{R}^{|\mathcal{S}\times \mathcal{A}|}$ and policy $\pi,$ we set $b_s(q,q_\kappa,\pi)$ to be the $\kappa$-greedy policy only when assured to have improvement:
\begin{align*}
b_s(q,q_\kappa,\pi) = \begin{cases}
a_\kappa(s) &\text{if } q(s,a_\kappa)\geq v^\pi(s),\\
a_{\text{1-step}}(s) &\text{else},
\end{cases}
\end{align*}
where $a_\kappa(s) \! \eqdef \! \argmax_aq_\kappa(s,a),~a_{\text{1-step}}(s) \! \eqdef \!  \argmax_a q(s,a),~\mbox{and }{v^\pi(s) \! = \!  \sum_a \pi(a\mid s)q(s,a)}.$

We respectively denote the state and state-action-pair visitation counters  after the $n$-th time-step by $\nu_n(s) \eqdef \sum_{k=1}^n \mathbbm{1}_{s=s_k}$ and $\phi_n(s,a) \eqdef \sum_{k=1}^n \mathbbm{1}_{(s,a)=(s_k,a_k)}$. The stepsize sequences $\mu_f(\cdot),\mu_s(\cdot)$ satisfy the common assumption (B2) in \cite{perkins2013asynchronous}, among which $\lim_{n\rightarrow\infty}\mu_s(n)/\mu_f(n)\rightarrow 0$. The second moments of $\{r_n\}$ are assumed to be bounded. Furthermore, let $\nu$ be some measure over the state space, s.t. $\forall s\in \mathcal{S},\ \nu (s)>0.$ Then, we assume to have a generative model $\mathbb{G}(\nu,\pi),$ using which we sample state $s\sim \nu$, sample  action $a\sim\pi(s)$, apply action $a$ and receive reward $r$ and next state $s'$. 

The fast-timescale update rules in lines \ref{eq: q pi update rule} and \ref{eq: q pi kappa update rule} can be jointly written as the sum of $H_\kappa^\pi(q,q_\kappa)$ (defined below) and a martingale difference noise.

\noindent\makebox[\textwidth][c]{
	\begin{minipage}[t]{0.85\textwidth}
		\begin{center}
			\begin{algorithm}[H]
				\caption{Two-Timescale Online $\kappa$-Policy-Iteration}
				\label{alg:async_kappaPI}
				\begin{algorithmic}[1]
					\STATE \textbf{initialize: }$\pi_0,q_0,q_{\kappa,0}.$
					\FOR{$n=0,\dots$}
					\STATE $s_n,a_n,r_n,s'_{n} \sim \mathbb{G}(\nu,\pi_n)$
					\STATE \textcolor{gray}{\# Fast-timescale updates}
					\STATE $\delta_n = r_{n}+\gamma v^\pi_n(s'_{n})-q_n(s_n,a_n)$
					\STATE $q_{n+1}(s_n,a_n) \gets q_n(s_n,a_n)+\mu_f(\phi_{n+1}(s_n,a_n))\delta_n $ \label{eq: q pi update rule}
					\STATE $\delta_{\kappa,n} = r_{n}+\gamma(1-\kappa) v^\pi_n(s'_{n})+\kappa\gamma\max_{a'} q_{\kappa,n}(s'_{n},a')-q_{\kappa,n}(s_n,a_n)$
					\STATE $q_{\kappa,n+1}(s_n,a_n) \gets q_{\kappa,n}(s_n,a_n)+\mu_f(\phi_{n+1}(s_n,a_n))\delta_{\kappa,n}$ \label{eq: q pi kappa update rule}
									\STATE \textcolor{gray}{\# Slow-timescale updates}
	
					\STATE $\pi_{n+1}(s_n) \gets
					 \pi_n(s_n)+\mu_s(\nu_{n+1}(s_n))(b_{s_n}(q_{n+1},q_{\kappa,n+1},\pi_n)-\pi_n(s_n))$
					\ENDFOR
					\STATE \textbf{return: } $\pi$
				\end{algorithmic}
			\end{algorithm}
		\end{center}
\end{minipage}}

\begin{defn}
	Let $q,q_\kappa \in \mathbb{R}^{|\mathcal{S}||\mathcal{A}|}.$ The mapping $H^\pi_\kappa:\mathbb{R}^{2|\mathcal{S}||\mathcal{A}|}\rightarrow \mathbb{R}^{2|\mathcal{S}||\mathcal{A}|}$ is defined as follows $\forall(s,a) \in \mathcal{S}\times \mathcal{A}$.
	\begin{align*}
	H_\kappa^\pi(q,q_\kappa)(s,a) \eqdef \begin{bmatrix}r(s,a)+\gamma \mathbb{E}_{s',a^\pi}q(s',a^\pi) \\ 
	r(s,a)+\gamma(1-\kappa) \mathbb{E}_{s',a^\pi}q(s',a^\pi) +\kappa\gamma \mathbb{E}_{s'}\max_{a'}q_\kappa(s',a') \end{bmatrix},
	\end{align*}
where $s'\sim P(\cdot \mid s,a), a^{\pi} \sim \pi(s')$.
\end{defn}

The following lemma shows that, given a fixed $\pi$, $H_\kappa^\pi$ is a contraction, equivalently to \cite[Lemma~5.3]{perkins2013asynchronous} (see Appendix~\ref{supp:online_kappa_PI_contraction} for the proof).
\begin{lemma}\label{lemma:online_kappa_PI_contraction}
	$H^\pi_\kappa$ is a $\gamma$-contraction in the max-norm. Its fixed point is $[\ q^\pi, q^{\pi}_{\kappa}\ ]^\top,$ as defined in~\eqref{eq: q pi def},~\eqref{eq: q pi kappa def}.
\end{lemma}

Finally, based on several intermediate results given in Appendix \ref{supp:online_kappa_PI_convergence} and relaying on Lemma \ref{lemma:online_kappa_PI_contraction}, we establish the convergence of Algorithm \ref{alg:async_kappaPI}. 
\begin{theorem} \label{theorem:online_kappa_PI_convergence}
	The coupled process $(q_n,q_{\kappa,n},\pi_n)$ in Algorithm \ref{alg:async_kappaPI} converges to the limit $(q^*,q^*,\pi^*)$, where $q^*$ is the optimal $q$-function and $\pi^*$ is the optimal policy.
\end{theorem}

For $\kappa=1$, the fast-timescale update rule in line \ref{eq: q pi kappa update rule} corresponds to that of $q$-learning \cite{watkins1992q}. For that $\kappa$, Algorithm \ref{alg:async_kappaPI} uses an estimated optimal $q$-function to update the current policy when improvement is assured. For $\kappa<1$, the estimated $\kappa$-dependent optimal $q$-function (see \eqref{eq: q pi kappa def}) is used, again with the `cautious' policy update. Moreover, Algorithm \ref{alg:async_kappaPI} combines an off-policy algorithm, i.e., $q$-learning, with an on-policy Actor-Critic algorithm. To the best of our knowledge, this is the first appearance of these two approaches combined in a single algorithm. 


\section{Approximate $\kappa$-Policy Iteration with Hard Updates}
\label{sec:app_kappa_greedy}
Theorem~\ref{thm:mixture_policies_and_monotonic_improvement} establishes the conditions required for guaranteed monotonic improvement of softly-updated multiple-step greedy policies. The algorithm in Section~\ref{sec:online_kappa_PI} then accounts for these conditions to ensure convergence. Contrarily, in this section, we derive and study algorithms that perform hard policy-updates. Specifically, we generalize the prominent Approximate Policy Iteration (API) \cite{munos2003error,farahmand2010error,lazaric2016analysis} and Policy Search by Dynamic Programming (PSDP) \cite{bagnell2004policy,scherrer2014approximate}. 
For both, we obtain performance guarantees that exhibit a tradeoff in the choice of $\kappa,$ with optimal performance bound achieved with $\kappa>0.$ That is, our approximate $\kappa$-generalized PI methods outperform the 1-step greedy approximate PI methods in terms of best known guarantees.


For the algorithms here we assume an oracle that returns a $\kappa$-greedy policy with some error. Formally, we denote by $\G_{\kappa,\delta,\nu}(v)$ the set of approximate $\kappa$-greedy policies w.r.t. $v,$ with $\delta$ approximation error under some measure $\nu$.

\begin{defn}[Approximate $\kappa$-greedy policy]
	\label{def:approximate_kappa_greedy}
	Let $v:\mathcal{S}\rightarrow\mathbb{R}$ be a value function, $\delta \geq 0$ a real number and $\nu$ a distribution over $\mathcal{S}$. A policy $\pi\in\G_{\kappa,\delta,\nu}(v)$ if $\nu T^{\pi}_\kappa v \geq \nu T_\kappa v - \delta. $
\end{defn}

Such a device can be implemented using existing approximate methods, e.g., Conservative Policy Iteration (CPI) \cite{kakade:02}, approximate PI or VI \cite{farahmand2010error}, Policy Search \cite{scherrer2014local}, or by having an access to an approximate model of the environment. The approximate $\kappa$-greedy oracle assumed here is less restrictive than the one assumed in \cite{efroni2018beyond}. There, a uniform error over states was assumed, whereas here, the error is defined w.r.t. a specific measure, $\nu$. For practical purposes, $\nu$ can be thought of as the initial sampling distribution to which the MDP is initialized. Lastly, notice that the larger $\kappa$ is, the harder it is to solve the surrogate $\kappa \gamma$-discounted MDP since the discount factor is bigger \cite{petrik2009biasing,strehl2009reinforcement,jiang2015dependence}; i.e., the computational cost of each call to the oracle increases. 

Using the concept of \emph{concentrability coefficients} introduced in \cite{munos2003error} (there, they were originally termed ``diffusion coefficients''), we follow the line of work in  \cite{munos2003error,munos2007performance,farahmand2010error, scherrer2014approximate,lazaric2016analysis} to prove our performance bounds. This allows a direct comparison of the algorithms proposed here with previously studied approximate 1-step greedy algorithms.
Namely, our bounds consist of concentrability coefficients $C^{(1)},C^{(2)},C^{(2,k)}$ and $C^{\pi^*(1)}$ from \cite{scherrer2014approximate,lazaric2016analysis}, as well as two new coefficients $C^{\pi^*}_{\kappa}$ and~$C^{\pi^*(1)}_\kappa$.

\begin{defn}[Concentrability coefficients \cite{scherrer2014approximate,lazaric2016analysis})]
	\label{def:concentrability_coeff} 
	Let $\mu,\nu$ be some measures over $\mathcal{S}.$ Let $\{c(i)\}_{i=0}^\infty$ be the sequence of the smallest values in $[1,\infty) \cup \{\infty \}$ such that for every $i,$ for all sequences of deterministic 
	stationary policies $\pi_1,\pi_2,..,\pi_i,$ $\mu \prod_{j=1}^{i} P^{\pi_j} \leq c(i)\nu$. Let	$C^{(1)}(\mu,\nu) = (1-\gamma)\sum_{i=0}^\infty \gamma^i c(i)$ and $C^{(2,k)}(\mu,\nu) = (1-\gamma)^{2}\sum_{i,j=0}^\infty \gamma^{i+j} c(i+j+k)$. For brevity, we denote $C^{(2,0)}(\mu,\nu)$ as $C^{(2)}(\mu,\nu).$ Similarly, let $\{c^{\pi^*}(i)\}_{i=0}^\infty$ be the sequence of the smallest values in $[1,\infty) \cup \{\infty \}$ such that for every $i,$ $\mu \left(P^{\pi^*}\right)^i \leq c^{\pi^*}(i)\nu$. Let $C^{\pi^*(1)}(\mu,\nu) = (1-\gamma)\sum_{i=0}^\infty \gamma^i c^{\pi^*}(i).$ 
\end{defn}

We now introduce two new concentrability coefficients suitable for bounding the worst-case performance of PI algorithms with approximate $\kappa$-greedy policies.


\begin{defn}[$\kappa$-Concentrability coefficients]
\label{def:kappa_concentrability_coeff} 
Let ${C^{\pi^*(1)}_\kappa(\mu,\nu) = \frac{\xi}{\gamma}C^{\pi^*(1)}(\mu,\nu)+(1-\xi)\kappa c(0)}$. Also, let  ${C^{\pi^*}_\kappa(\mu,\nu) \in [1,\infty) \cup \{\infty \}}$ be the smallest value s.t.
 $d^{\pi^*}_{\kappa,\mu} \leq C^{\pi^*}_\kappa(\mu,\nu) \nu,$ where 
	${d^{\pi^*}_{\kappa,\mu}=(1-\xi)\mu(I-\xi D^{\pi^*}_{\kappa}P^{\pi^*})^{-1}}$ is a probability measure and ${D_\kappa^{\pi}=(1-\kappa\gamma)(I-\kappa\gamma P^\pi)^{-1}}$ is a stochastic matrix.
\end{defn}


In the definitions above, $\nu$ is the measure according to which the approximate improvement is guaranteed, while $\mu$ specifies the distribution on which one measures the loss ${\mathbb{E}_{s \sim \mu}[v^*(s)-v^{\pi_k}(s)]=\mu(v^*-v^{\pi_k})}$ that we wish to bound.
From Definition~\ref{def:kappa_concentrability_coeff}  it holds that ${C^{\pi^*}_{\kappa=0}(\mu,\nu) =C^{\pi^*}(\mu,\nu)};$ the latter was previously defined in, e.g, \cite[Definition~1]{scherrer2014approximate}. 

Before giving our performance bounds, we first study the behavior of the coefficients appearing in them. The following lemma sheds light on the behavior of $C^{\pi^*}_\kappa(\mu,\nu).$ Specifically, it shows that under certain constructions, $C^{\pi^*}_\kappa(\mu,\nu)$ decreases\footnote{A smaller coefficient is obviously better. The best value for any concentrability coefficient is 1.} as $\kappa$ increases (see proof in Appendix~\ref{supp:general_C_improvement}).
\begin{lemma}\label{lemma:general_C_improvement}
	Let $\nu(\alpha)=(1-\alpha)\nu + \alpha \mu$. Then, for all $\kappa'>\kappa$, there exists $\alpha^*\in(0,1)$ such that $C^{\pi^*}_{\kappa'}(\mu,\nu(\alpha^*)) \leq C^{\pi^*}_\kappa(\mu,\nu).$ The inequality is strict for $C^{\pi^*}_\kappa(\mu,\nu)>1$. For $\mu=\nu$ this implies that $C^{\pi^*}_{\kappa}(\nu,\nu)$ is a decreasing function of $\kappa$. 
\end{lemma}


Definition~\ref{def:kappa_concentrability_coeff} introduces two coefficients with which we shall derive our bounds. Though traditional arithmetic relations between them do not exist, they do comply to some notion of ordering.

\begin{remark}[Order of concentrability coefficients]
	\label{remark:order_of_coefficents}
	In \cite{scherrer2014approximate}, an order between the concentrability coefficients was introduced: a coefficient $A$ is said to be strictly better than $B$ --- a relation we denote with $A \prec B$ --- if and only if i) $B<\infty$ implies $A<\infty$ and ii) there exists an MDP for which $A<\infty$ and $B=\infty$. Particularly, it was  argued that 
\begin{align*}
&C^{\pi^*}(\mu,\nu) \prec C^{\pi^*(1)}(\mu,\nu) \prec C^{(1)}(\mu,\nu) \prec C^{(2)}(\mu,\nu),\ \mbox{and}\\
&C^{(2,k_1)}(\mu,\nu)\prec C^{(2,k_2)}(\mu,\nu)\ \mbox{if}\ k_2<k_1.
\end{align*}	
In this sense, $C^{\pi^*(1)}_\kappa(\mu,\nu)$ is analogous to $C^{\pi^*(1)}(\mu,\nu)$, while its definition might suggest improvement as $\kappa$ increases. Moreover, combined with the fact that $C^{\pi^*}_\kappa(\mu,\nu)$ improves as $\kappa$ increases, as Lemma~\ref{lemma:general_C_improvement} suggests, $C^{\pi^*}_\kappa(\mu,\nu)$ is better than all previously defined concentrability coefficients. 
\end{remark}

\subsection{$\kappa$-Approximate Policy Iteration}
 A natural generalization of API \cite{munos2003error,scherrer2014approximate,lazaric2016analysis} to the multiple-step greedy policy is $\kappa$-API, as given in Algorithm~\ref{alg:kappaAPI}. In each of its iterations, the policy is updated to the approximate $\kappa$-greedy policy w.r.t. $v^{\pi_{k-1}}$; i.e,  a policy from the set $\G_{\kappa,\delta,\nu}(v^{\pi_{k-1}})$. 
 
\noindent\makebox[\textwidth][c]{
\begin{minipage}[t]{0.45\textwidth}
\centering
 \begin{algorithm}[H]
 \begin{algorithmic}
    \STATE \textbf{initialize } $\kappa\in[0,1],\nu,\delta,v^{\pi_0}$
	\STATE $v\gets v^{\pi_0}$    
    \FOR{$\ k=1,..$}
    	\STATE{$\pi_k \gets \G_{\kappa,\delta,\nu}(v)$}
    	\STATE{$v \gets v^{\pi_k}$}
    \ENDFOR
    \RETURN $\pi$
    \end{algorithmic}
       \caption{$\kappa$-API}
     \label{alg:kappaAPI}
      	\vspace{0.45cm}
  \end{algorithm}
\end{minipage}
\hspace{0.4cm}
\begin{minipage}[t]{0.45\textwidth}
\centering
 \begin{algorithm}[H]
    \caption{$\kappa$-PSDP}
    \label{alg:kappaPSDP}
	\begin{algorithmic}
    \STATE \textbf{initialize } $\kappa\in[0,1],\nu,\delta,v^{\pi_0}, \Pi=[\ ]$
    \STATE $v \gets v^{\pi_0}$
    \FOR{$\ k=1,..$}
    	\STATE{$\pi_k \gets \G_{\kappa,\delta,\nu}(v)$}
    	\STATE{$v \gets T_\kappa^{\pi_k}v$}
    	\STATE{$\Pi\gets $Append$(\Pi,\pi_k)$}
    \ENDFOR
    \RETURN $\Pi$    
    \end{algorithmic}
  \end{algorithm}
\end{minipage}}

The following theorem gives a performance bound for $\kappa$-API (see proof in Appendix~\ref{supp:kappa_API}), with
\begin{align*}
&\Ckapi(\mu,\nu) = (1-\kappa)^{2} C^{(2)}(\mu,\nu) + (1-\gamma)\kappa \left((1-\kappa)C^{(1)}(\mu,\nu) + (1-\gamma\kappa)C^{\pi^*(1)}_{\kappa}(\mu,\nu)\right),\\
&\Ckapi^{(k,1)}(\mu,\nu)=(1-\kappa\gamma)\left(\kappa(1-\kappa\gamma)C_\kappa^{\pi^*}(\mu,\nu)+(1-\kappa)^2C^{(1)}(\mu,\nu))\right),\\
&\Ckapi^{(k,2)}(\mu,\nu) = (1-\kappa)\kappa \left((1-\gamma) C^{(1)}(\mu,\nu)+g(\kappa)(1-\kappa)\gamma^{k} C^{(2,k)}(\mu,\nu)\right), 
\end{align*}  
where $g(\kappa)$ is a bounded function for $\kappa \in [0,1].$

\begin{theorem}\label{theorem:kappa_API}
Let $\pi_k$ be the policy at the $k$-th iteration of  $\kappa$-API and $\delta$ be the error as defined in Definition~\ref{def:approximate_kappa_greedy}. Then 
\begin{align*}
\mu(v^{*}-v^{\pi_k})\leq  \frac{\Ckapi(\mu,\nu)}{(1-\gamma)^2} \delta   + \xi^{k} \frac{R_{\max}}{1-\gamma}.
\end{align*}
Also, let $k=\left\lceil  \frac{\log{\frac{R_{max}}{\delta(1-\gamma)}}}{1-\xi} \right\rceil.$ Then
$\mu(v^{*}-v^{\pi_{k}})\leq \frac{\Ckapi^{(k,1)}(\mu,\nu)}{(1-\gamma)^2} \log\left( \frac{R_{\max}}{(1-\gamma)\delta}\right)\delta +\frac{\Ckapi^{(k,2)}(\mu,\nu)}{(1-\gamma)^2} \delta + \delta. 
$
\end{theorem}
For brevity, we now discuss the first part of the statement; the same insights are true for the second as well.
The bound for the original API is restored for the 1-step greedy case of $\kappa=0$, i.e, ${\mu(v^{*}-v^{\pi_k})\leq \frac{C^{(2)}(\mu,\nu)}{(1-\gamma)^2}\delta+\frac{\gamma^k R_{\max}}{1-\gamma}}$ \cite{scherrer2014approximate,lazaric2016analysis}. As in the case of API, our bound consists of a fixed approximation error term and a geometrically decaying term.  As for the other extreme, $\kappa=1,$ we first remind that in the non-approximate case, applying $T_{\kappa=1}$ amounts to solving the original $\gamma$-discounted MDP in a single step \cite[Remark~4]{efroni2018beyond}. In the approximate setup we investigate here, this results in the vanishing of the second, geometrically decaying term, since $\xi=0$ for $\kappa=1$. We are then left with a single constant approximation error:  $\mu(v^{*}-v^{\pi_k})\leq c(0)\delta.$ Notice that $c(0)$ is independent of $\pi^*$ (see Definition~\ref{def:concentrability_coeff}). It represents the mismatch between $\mu$ and $\nu$ \cite{kakade:02}. 
	



Next, notice that, by definition (see Definition \ref{def:concentrability_coeff}), $C^{(2)}(\mu,\nu) > (1-\gamma)^2 c(0);$ i.e., $\frac{C^{(2)}(\mu,\nu)}{(1-\gamma)^2} \delta > c(0) \delta.$ Given the discussion above, we have that the $\kappa$-API performance bound is \emph{strictly} smaller with $\kappa=1$ than with $\kappa=0.$ Hence, the bound suggests that $\kappa$-API is strictly better than the original API for $\kappa=1.$ Since all expressions there are continuous, this behavior does not solely hold point-wise.
%



\begin{remark}[Performance tradeoff]
\label{remark:kappa_tradeoff}
Naively, the above observation would lead to the choice of $\kappa=1$. However, it is reasonable to assume that $\delta$, the error of the $\kappa$-greedy step, itself depends on $\kappa,$ i.e, $\delta \equiv \delta(\kappa)$. The general form of such dependence is expected to be monotonically increasing: as the effective horizon of the surrogate $\kappa \gamma$-discounted MDP becomes larger, its solution is harder to obtain (see Remark~\ref{rem: applying T kappa}). Thus, Theorem~\ref{theorem:kappa_API} reveals a performance tradeoff as a function of $\kappa$. 
\end{remark}

\subsection{$\kappa$-Policy Search by Dynamic Programming}

We continue with generalizing another approximate PI method -- PSDP \cite{bagnell2004policy,scherrer2014approximate}. We name it $\kappa$-PSDP and introduce it in Algorithm~\ref{alg:kappaPSDP}. This algorithm updates the policy differently from $\kappa$-API. However, similarly to $\kappa$-API, it uses hard updates. We will show this algorithm exhibits better performance than any other previously analyzed approximate PI method \cite{scherrer2014approximate}. 

The $\kappa$-PSDP algorithm, unlike $\kappa$-API, returns a \emph{sequence of deterministic policies}, $\Pi$. Given this sequence, we build a single, non-stationary policy by successively running $N_k$ steps of $\Pi[k]$, followed by $N_{k-1}$ steps of $\Pi[k-1]$, etc, where $\{N_i\}_{i=1}^k$ are i.i.d. geometric random variables with parameter $1-\kappa$.  Once this process reaches $\pi_0$, it runs $\pi_0$ indefinitely. We shall refer to this non-stationary policy as $\sigma_{\kappa,k}$. Its value $v^{\sigma_{\kappa,k}}$ can be seen to satisfy 
\begin{align*}
v^{\sigma_{\kappa,k}}= T^{\Pi[k]}_{\kappa}T^{\Pi[k-1]}_{\kappa}\dots T^{\Pi[1]}_{\kappa} v^{\pi_0}.
\end{align*}


This algorithm follows PSDP from \cite{scherrer2014approximate}. Differently from it, the 1-step improvement is generalized to the $\kappa$-greedy improvement and the non-stationary policy behaves randomly. The following theorem gives a performance bound for it
 (see proof in Appendix~\ref{supp:kappa_PSDP}).

\begin{theorem}\label{theorem:kappa_PSDP}
	Let $\sigma_{\kappa,k}$ be the policy at the $k$-th  iteration of $\kappa$-PSDP and $\delta$ be the error as defined in Definition~\ref{def:approximate_kappa_greedy}.
 Then
	\begin{align*}
	\mu (v^* - v^{\sigma_{\kappa,k}}) \le  \frac{C_{\kappa}^{\pi^*(1)}(\mu,\nu)}{1-\xi} \delta + \xi^k \frac{R_{\max}}{1-\gamma}.
	\end{align*}
Also, let $k=\left\lceil  \frac{\log{\frac{R_{max}}{\delta(1-\gamma)}}}{1-\xi} \right\rceil.$ Then
$
	\mu (v^* - v^{\sigma_{\kappa,k}}) \le   \frac{C_{\kappa}^{\pi^*}(\mu,\nu)}{(1-\xi)^2} \log\left( \frac{R_{\max}}{(1-\gamma)\delta}\right)\delta + \delta.
$
\end{theorem}

Compared to $\kappa$-API from the previous section, the $\kappa$-PSDP bound consists of a different fixed approximation error and a shared geometrically decaying term. Regarding the former, notice that ${C_{\kappa}^{\pi^*(1)}(\mu,\nu) \prec \Ckapi(\mu,\nu)},$ using the notation from Remark~\ref{remark:order_of_coefficents}. This suggests that $\kappa$-PSDP is strictly better than  $\kappa$-API in the metrics we consider, and is in line with the comparison of the original API to the original PSDP given in \cite{scherrer2014approximate}.

Similarly to the previous section, we again see that substituting $\kappa=1$ gives a tighter bound than $\kappa=0.$ The reason is that $\frac{C^{\pi^*(1)}(\mu,\nu)}{1-\gamma}\delta > c(0)\delta$, by definition (see Definition \ref{def:concentrability_coeff});
 i.e., we have that $\kappa$-PSDP is generally better than PSDP. Also, contrarily to $\kappa$-API, here we directly see the performance improvement as $\kappa$ increases due to the decrease of $C_{\kappa}^{\pi^*}$ prescribed in Lemma~\ref{lemma:general_C_improvement}, for the construction given there. Moreover, the $\kappa$ tradeoff discussion in Remark~\ref{remark:kappa_tradeoff} applies here as well.

An additional advantage of this new algorithm over PSDP is reduced space complexity. This can be seen from the $1-\xi$ in the denominator in the choice of $k$ in the second part of Theorem~\ref{theorem:kappa_PSDP}. It shows that, since $\xi$ is a strictly decreasing function of $\kappa$, better performance is guaranteed with significantly fewer iterations by increasing $\kappa$. Since the size of stored policy $\Pi$ is linearly dependent on the number of iterations, larger $\kappa$ improves space efficiency. 



\section{Discussion and Future Work}
\label{sec:conclusions}
In this work, we introduced and analyzed online and approximate PI methods, generalized to the $\kappa$-greedy policy, an instance of a multiple-step greedy policy. Doing so, we discovered two intriguing properties compared to the well-studied 1-step greedy policy, which we believe can be impactful in designing state-of-the-art algorithms. First,
 successive application of multiple-step greedy policies with a soft, stepsize-based update does not guarantee improvement; see Theorem~\ref{thm:mixture_policies_and_monotonic_improvement}. To mitigate this caveat, we designed an online PI algorithm with a `cautious' improvement operator; see Section~\ref{sec:online_kappa_PI}. 
 
 The second property we find intriguing stemmed from analyzing $\kappa$ generalizations of known approximate hard-update PI methods. In Section~\ref{sec:app_kappa_greedy}, we revealed a performance tradeoff in $\kappa,$ which can be interpreted as a tradeoff between short-horizon bootstrap bias and long-rollout variance. This corresponds to the known $\lambda$ tradeoff in the famous TD($\lambda$).

The two characteristics above lead to new compelling questions. The first regards improvement operators: would a non-monotonically improving PI scheme necessarily not converge to the optimal policy?  Our attempts to generalize existing proof techniques to show convergence in such cases have fallen behind. Specifically, in the online case, Lemma~5.4 in \cite{konda1999actor} does not hold with multiple-step greedy policies. Similar issues arise when trying to form a $\kappa$-CPI algorithm via, e.g., an attempt to generalize Corollary~4.2 in \cite{kakade:02}.  Another research question regards the choice of the parameter $\kappa$ given the tradeoff it poses. One possible direction for answering it could be investigating the concentrability coefficients further and attempting to characterize them for specific MDPs, either theoretically or via estimation. Lastly, a next indisputable step would be to empirically evaluate implementations of the algorithms presented here.   


\section*{Acknowledgments}
This work was partially funded by the Israel Science Foundation under contract 1380/16.

\bibliography{approximate_kappa_greedy_policy}
\bibliographystyle{plain}

\newpage
\appendix


\section{Proof of Theorem \ref{thm:mixture_policies_and_monotonic_improvement}}\label{supp:mixture_policies_and_monotonic_improvement}

We start with a  generalization of a useful lemma; its original version appeared in, e.g., \cite[Lemma~10]{scherrer2013improved}.
\begin{lemma}\label{lemma:kappa_value_difference}
	Let $v$ be a value function, $\pi$ a policy, and $\kappa \in [0,1]$. Then
	\begin{align*}
	T_\kappa^{\pi}v - v = (I-\kappa\gamma P^{\pi})^{-1}(T^{\pi}v - v).
	\end{align*}
\end{lemma}

\begin{proof}
The proof is a straightforward generalization of the proof in \cite[Lemma~10]{scherrer2013improved}, and \cite[Remark~6.1]{kakade:02}.
\begin{align*}
T^{\pi}_\kappa v - v &= (I-\kappa\gamma P^{\pi})^{-1}(r^{\pi} + (1-\kappa)\gamma P^{\pi}v)-v\\
&= (I-\kappa\gamma P^\pi)^{-1}(r^{\pi} + (1-\kappa)\gamma P^{\pi}v-(I-\kappa\gamma P^\pi)v)\\
&= (I-\kappa\gamma P^\pi)^{-1}(r^{\pi} + \gamma P^{\pi}v - v )\\
&= (I-\kappa\gamma P^\pi)^{-1}(T^\pi v - v ).
\end{align*}
\end{proof}

This elementary lemma relates the `$\kappa$-advantage' to the 1-step advantage and is useful to prove Theorem \ref{thm:mixture_policies_and_monotonic_improvement} and some following results.

First, since $\pi(\alpha,\kappa)= (1-\alpha)\pi+\alpha \pi_\kappa$, we have that
\begin{align*}
&P^{\pi(\alpha,\kappa)} = (1-\alpha)P^{\pi}+\alpha P^{\pi_\kappa},\\
&r^{\pi(\alpha,\kappa)} = (1-\alpha)r^{\pi}+\alpha r^{\pi_\kappa};
\end{align*}
thus, since $v^\pi$ is the fixed-point of $T^\pi,$
\begin{align}
\label{eq: split}
T^{\pi(\alpha,\kappa)}v^\pi = (1-\alpha)T^{\pi}v^\pi+\alpha T^{\pi_\kappa}v^\pi = (1-\alpha) v^\pi+\alpha T^{\pi_\kappa}v^\pi.
\end{align}
Using this, we now prove the first statement of Theorem \ref{thm:mixture_policies_and_monotonic_improvement}.
\begin{align}
v^{\pi(\alpha,\kappa)}-v^{\pi} &= (I-\gamma P^{\pi(\alpha,\kappa)})^{-1}(T^{\pi(\alpha,\kappa)}v^\pi-v^\pi)\nonumber\\
&= \alpha(I-\gamma P^{\pi(\alpha,\kappa)})^{-1}(T^{\pi_\kappa}v^\pi-v^\pi)\nonumber\\
&= \alpha(I-\gamma P^{\pi(\alpha,\kappa)})^{-1}(I-\kappa\gamma P^{\pi_\kappa})(I-\kappa\gamma P^{\pi_\kappa})^{-1}(T^{\pi_\kappa}v^\pi-v^\pi)\nonumber\\
&= \alpha(I-\gamma P^{\pi(\alpha,\kappa)})^{-1}(I-\kappa\gamma P^{\pi_\kappa})(T_\kappa^{\pi_\kappa}v^{\pi}-v^{\pi})\nonumber\\
&= \alpha(I-\gamma P^{\pi(\alpha,\kappa)})^{-1}(I-\gamma P^{\pi(\alpha,\kappa)}+\gamma(P^{\pi(\alpha,\kappa)}-\kappa P^{\pi_\kappa}))(T_\kappa^{\pi_\kappa}v^{\pi}-v^{\pi})\nonumber\\
&= \alpha(I+\gamma(I-\gamma P^{\pi(\alpha,\kappa)} )^{-1}((1-\alpha)P^{\pi}+(\alpha-\kappa) P^{\pi_\kappa})(T_\kappa^{\pi_\kappa}v^{\pi}-v^{\pi}). \label{eq: v pi alpha kappa diff}
\end{align}

For the first relation we use Lemma~\ref{lemma:kappa_value_difference} with $\kappa=1$ and the fact that, by definition, ${T^{\pi(\alpha,\kappa)}_{\kappa=1} v^{\pi(\alpha,\kappa)} = v^{\pi(\alpha,\kappa)}}$. For the second relation we use \eqref{eq: split}, for the fourth we again use Lemma~\ref{lemma:kappa_value_difference}, and for the last relation we use that $P^{\pi(\alpha,\kappa)}-\kappa P^{\pi_\kappa} = (1-\alpha)P^\pi+(\alpha-\kappa) P^{\pi_\kappa}$.

Next, we show that for $\alpha\geq \kappa,$ all terms in \eqref{eq: v pi alpha kappa diff} are component-wise bigger than or equal to zero. First, using a Taylor expansion, ${(I-\gamma P^{\pi(\alpha,\kappa)})^{-1}=\sum_t \gamma^t (P^{\pi(\alpha,\kappa)})^t\geq 0}$ component-wise,  since it is a weighted sum of transition matrices with positive weights. The same applies for ${(1-\alpha)P^{\pi}+(\alpha-\kappa) P^{\pi_\kappa},}$ when $\alpha\geq \kappa.$  Thus, for $\alpha\geq \kappa$, ${(I+\gamma(I-\gamma P^{\pi(\alpha,\kappa)})^{-1}((1-\alpha)P^{\pi}+(\alpha-\kappa) P^{\pi_\kappa})\geq 0}$ component-wise. Lastly, since ${\pi_\kappa\in \G_\kappa(v^{\pi})}$,
$
{v^{\pi} = T_\kappa^{\pi} v^{\pi} \leq   T_\kappa v^{\pi} = T^{\pi_\kappa}_\kappa v^{\pi},}
$
with equality holding if and only if $v^{\pi}=v^*$ \cite[Lemma~3]{efroni2018beyond}. Thus, $T_\kappa^{\pi_\kappa}v^{\pi}-v^{\pi}\geq 0$. This concludes the proof for the first statement, for the $\kappa$-greedy policy.

For the $\kappa$-greedy policy part of the proof for the second statement, we now provide more details on the counterexample presented in Section~\ref{sec:monotonicity_and_soft}. For convenience, we bring the MDP example here again in Fig.~\ref{fig:thm_mixture_policies}.
Consider the mixture of the ``hesitant'' and ``confident'' policies: $\pi(\alpha,\kappa=1)=(1-\alpha)\pi_0+\alpha\pi(\alpha,\kappa=1)$. It can be shown that its value is
\begin{align*}
&v^{\pi(\alpha,\kappa=1)}(s_0)=\frac{\gamma\alpha}{1-\gamma(1-\alpha)}v^{\pi(\alpha,\kappa=1)}(s_1),\\
&v^{\pi(\alpha,\kappa=1)}(s_1)=\gamma\frac{-c(1-\alpha)+\alpha}{1-\gamma}.
\end{align*}
Thus, we deduce that for any $\alpha\in(0,1)$ and 
\begin{equation}
\label{eq: c lower bound app}
c>\frac{\alpha}{1-\alpha},
\end{equation}
$v^{\pi(\alpha,\kappa=1)}(s_0)<v^{\pi}(s_0)=0$, i.e, the mixture policy, $\pi(\alpha,\kappa=1)$, is not strictly better then $\pi_0.$  

 We now find the conditions to ensure that the $\kappa$-greedy policy w.r.t. $v^{\pi_0}$ is the optimal policy; this will generalize the above construction, made for $\kappa=1,$ to any $\kappa\in[0,1].$ Observe that for any $c>0$ and $\kappa$ it holds that ${\pi_\kappa(s_1)=a_1=\pi^*(s_1)},$ where $\pi_\kappa \in \G_\kappa(v^{\pi_0}).$ Thus, we solely need to consider the policy which is different than $\pi^*$ at state $s_0$, $\tilde{\pi}(s_0)=a_0 \neq \pi^*(s_0)$ and $\tilde{\pi}(s_1)=\pi^*(s_1)$. To find which condition ensures the $\kappa$-greedy policy w.r.t. $v^{\pi_0}$ is $\pi^*$ (and not $\tilde{\pi}$), we require
\begin{align}
\label{eq:condition_kappa_optimal_policy}
T^{\pi^*}_\kappa v^{\pi_0} (s_0) \geq T^{\tilde{\pi}}_\kappa v^{\pi_0} (s_0).
\end{align}
Satisfying this condition insures that $\pi^*\in\G_\kappa(v^{\pi_0}).$ By definition,
\begin{align}
&T^{\pi^*}_\kappa v^{\pi_0} (s_0) 
= \mathbb{E}^{\pi^*}\left[\sum_{t} (\kappa\gamma)^t (r(s_t,\pi^*(s_t))+\gamma(1-\kappa) v^{\pi_0}(s_{t+1})\mid s_{t=0}=s_0\right]\nonumber \\
=&(\kappa\gamma)^0\left(\gamma(1-\kappa)v^{\pi_0}(s_1))\right) +(\kappa\gamma)^1  \left(\gamma(1-\kappa)v^{\pi_0}(s_2)\right) + \sum_{t=2}^{\infty} (\kappa\gamma)^t(1+v^{\pi_0}(s_2))\nonumber\\ =&(\kappa\gamma)^0\left(\gamma(1-\kappa)(-\frac{\gamma c}{1-\gamma})\right) +(\kappa\gamma)^1  \left( \gamma(1-\kappa)\frac{1}{1-\gamma}\right) + \sum_{t=2}^{\infty} (\kappa\gamma)^t(1+\gamma(1-\kappa)\frac{1}{1-\gamma})\nonumber\\
=&\gamma(1-\kappa)(-\frac{\gamma c}{1-\gamma}) +\kappa\gamma \frac{\gamma}{1-\gamma}.\label{eq:condition_kappa_optimal_policy_LHS}
\end{align}
Similarly, and since $\tilde{\pi}(s_0)=a_0$, we have that
\begin{align}
T^{\tilde{\pi}}_\kappa v^{\pi_0} (s_0) = 0 \label{eq:condition_kappa_optimal_policy_RHS}
\end{align}
Plugging \eqref{eq:condition_kappa_optimal_policy_LHS} and \eqref{eq:condition_kappa_optimal_policy_RHS} into \eqref{eq:condition_kappa_optimal_policy}, we get the condition
\begin{equation}
\label{eq: c upper bound app}
c\leq\frac{\kappa}{1-\kappa}.
\end{equation}
To finalize the counterexample and show that strict policy improvement is not guaranteed, we choose $c$ such that both \eqref{eq: c lower bound app} and \eqref{eq: c upper bound app} are satisfied. Such feasible choice exists when $\alpha<\kappa,$ due to the monotonicity of $\frac{x}{1-x}.$

The monotonic improvement of $\pi(\alpha,h)$ for $\alpha=1$ was proved in \cite[Lemma~1]{efroni2018beyond}. To build the counter example, again consider the Tightrope MDP. Let $\pi_0$ be the `hesitant' policy. For any $\gamma\in(0,1)$, $h>1$, it holds that $\pi^*\in \G_h(v^{\pi_0})$. Thus, it suffices to satisfy \eqref{eq: c lower bound app} alone to show that $\pi(\alpha,h)=(1-\alpha)\pi_0+\alpha \pi^*$ is not monotonically better then $\pi$. Large enough $c$ value ensures that.
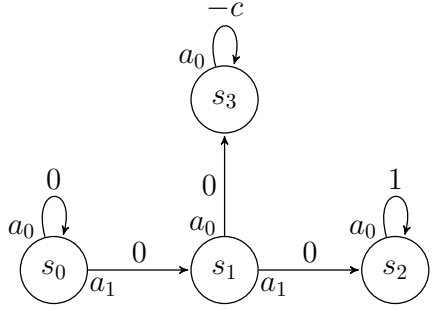
\begin{figure}
\centering
\resizebox{2.3in}{!}{
\begin{tikzpicture}[->,>=stealth',shorten >=1pt,auto,node distance=2.8cm,
                    semithick, state/.style={circle, draw, minimum size=1.1cm}]
\tikzstyle{every state}=[thick]
]

\node[state] (S0) {\Large $s_0$};
\node[state] (S1) [right of=S0] {\Large $s_1$};
\node[state] (S2) [right of=S1] {\Large $s_2$};
\node[state] (S3) [above of=S1] {\Large $s_3$};

\path (S0) edge  [loop above] node[pos=0.05,left]{\Large $a_0$} node {\Large $0$} (S0)
           edge     node[pos=0.15,below ]{\Large $a_1$}         node {\Large $0$} (S1)
	  (S1) edge     node[pos=0.1,left]{\Large $a_0$} node {\Large $0$} (S3)
           edge     node[pos=0.15,below ]{\Large $a_1$}         node {\Large $0$} (S2)
      (S2) edge [loop above] node[pos=0.05,left]{\Large $a_0$} node {\Large $1$} (S2)
      (S3) edge  [loop above] node[pos=0.05,left]{\Large $a_0$} node {\Large $-c$} (S3);
           
\end{tikzpicture}
}
\caption{The Tightrope Walking MDP used in the proof of Theorem~\ref{thm:mixture_policies_and_monotonic_improvement}. This class of MDPs is parametrized by $c>0$.}
\label{fig:thm_mixture_policies}
\end{figure}

\section{Proof of Lemma \ref{lemma:online_kappa_PI_contraction}}\label{supp:online_kappa_PI_contraction}
We start by showing the contraction property of $H^\pi_\kappa$. Let $(s,a)$ be a fixed state-action pair, $Q_1,Q_2 \in \mathbb{R}^{2|\mathcal{S}\times\mathcal{A}|}$. For any state-action pair $(s,a),$ $Q_i(s,a)$ is a two-component vector.  We denote its first component by $q_i(s,a)$ and its second component by $q_{i,\kappa}(s,a)$. See that 
\begin{align}
||q_{1}-q_{2}||_\infty &\leq ||Q_{1}-Q_{2}||_\infty \label{eq:q_inf_norm}, \\
||q_{1,\kappa}-q_{2,\kappa}||_\infty &\leq ||Q_{1}-Q_{2}||_\infty \label{eq:q_kappa_inf_norm} .
\end{align}
Taking a component-wise absolute value, we have that
\begin{align*}
&|H^\pi_\kappa Q_1-H^\pi_\kappa Q_2|(s,a)\\
=&|H^\pi_\kappa (q_1,q_{1,\kappa})-H^\pi_\kappa (q_2,q_{2,\kappa})|(s,a)\\
=&\gamma\begin{bmatrix} |\mathbb{E}_{s',a^{\pi}}\left[q_1(s',a^{\pi}))-q_2(s',\pi(s'))\right]| \\
			    |(1-\kappa) \mathbb{E}_{s',a^{\pi}}\left[q_1(s',a^{\pi})-q_2(s',a^{\pi}))\right] +
			     \kappa \mathbb{E}_{s'}[\max_{a'}q_{1,\kappa}(s',a')- \max_{a'}q_{2,\kappa}(s',a')]| \end{bmatrix},
\end{align*}
where $s'\sim P(\cdot \mid s,a), a^{\pi} \sim \pi(s')$.

Let us focus on the first component of the above vector. We have that
\begin{align*}
&\gamma |\mathbb{E}_{s',a^{\pi}}\left[q_1(s',a^{\pi})-q_2(s',a^{\pi})\right]|\leq \gamma ||q_1-q_2||_\infty \leq\gamma ||Q_1-Q_2||_\infty,
\end{align*}
where we used the standard bound, $|\mathbb{E}[X]|\leq ||X||_\infty$ and \eqref{eq:q_inf_norm}.
Similarly, for the second component, we have that
\begin{align*}
&\gamma\left|\left((1-\kappa) \mathbb{E}_{s',a^{\pi}}\left[q_1(s',a^{\pi})-q_2(s',a^{\pi})\right] +
					  \kappa \mathbb{E}_{s',a}[\max_{a'}q_{1,\kappa}(s',a')- \max_{a'}q_{2,\kappa}(s',a')]\right)\right|\\
\leq&\gamma\left((1-\kappa) |\mathbb{E}_{s',a^{\pi}}\left[q_1(s',a^{\pi})-q_2(s',a^{\pi})\right]| +
					   \kappa \mathbb{E}_{s',a}[|\max_{a'}q_{1,\kappa}(s',a')- \max_{a'}q_{2,\kappa}(s',a')|]\right)\\	
\leq&\gamma\left((1-\kappa) |\mathbb{E}_{s',a^{\pi}}\left[q_1(s',a^{\pi})-q_2(s',a^{\pi})\right]|  +
					   \kappa \mathbb{E}_{s',a'}[\max_{a'}|q_{1,\kappa}(s',a')-q_{2,\kappa}(s',a')|]\right)\\		
\leq&\gamma\left((1-\kappa) ||q_1-q_2||_\infty +
					   \kappa ||q_{1,\kappa}-q_{2,\kappa}||_\infty \right)\\		
\leq&\gamma\left((1-\kappa) ||Q_1-Q_2||_\infty +
					   \kappa ||Q_1-Q_2||_\infty\right) = \gamma||Q_1-Q_2||_\infty,
\end{align*}
where for the first relation we used the triangle inequality, for the second we used the standard bound $|\max_{x\in \mathcal{X}}f(x)-\max_{x\in \mathcal{X}}g(x)|\leq \max_{x\in \mathcal{X}}| f(x)-g(x)|$, for the third we used the bound $|\mathbb{E}[X]|\leq ||X||_\infty,$ and for the last \eqref{eq:q_inf_norm}-\eqref{eq:q_kappa_inf_norm}.
 
From the above we get that
\begin{align*}
||H^\pi_\kappa Q_1-H^\pi_\kappa Q_2||_\infty \leq \gamma||Q_1-Q_2||_\infty;
\end{align*}
i.e., the operator $H^\pi_\kappa$ is a $\gamma$ contraction mapping in the max-norm.

It is clear that the fixed point of the first component is $q^\pi$. The fixed point of the second component is the fixed point of the optimal Bellman operator of the $\kappa\gamma$-discounted, reward shaped, surrogate MDP (see Remark \ref{rem: applying T kappa}). Its solution is, by construction, $q_\kappa^\pi$ (see \eqref{eq: q pi kappa def}).

\section{Proof of Theorem~\ref{theorem:online_kappa_PI_convergence}}\label{supp:online_kappa_PI_convergence}

The proof of Theorem~\ref{theorem:online_kappa_PI_convergence} follows the proof in \cite[Section~5.1]{perkins2013asynchronous}, with several generalizations given below.

\subsection{Lipschitzness of the Slow Time Scale Fixed-Point}
Before following the main lemmas in \cite{perkins2013asynchronous} and showing they hold for Online $\kappa$-PI (Algorithm~\ref{alg:async_kappaPI}), we shall show that the solution of the fast-time scale ODE (found using a fixed-point argument), $[q^\pi,q^\pi_\kappa]$,  is Lipschitz-continuous in the slow time-scale iterate, $\pi$.
 
\begin{lemma}
	Let $\pi: \mathcal{S} \times \mathcal{A} \rightarrow [0,1]$ be a stochastic policy. For any $\pi_1,\pi_2$ and $q_1,q_2 \in \mathbb{R}^{|\mathcal S \times \mathcal{A}|}$, let
	\begin{align*}
	&||\pi_1-\pi_2||_{\infty}\eqdef \max_{s} \sum_a |\pi_1(a \mid s)-\pi_2(a \mid s)|,\\
	&||q_1-q_2||_{\infty}\eqdef \max_{s,a} |q_1(s,a)-q_2(s,a)|.
	\end{align*}
	Then $q^\pi$ and $q^\pi_\kappa$ are Lipschitz-continuous in $\pi$ in the max-norm; i.e., 
	\begin{align*}
	&||q^{\pi_1}-q^{\pi_2}||_{\infty}\leq L_a ||\pi_1-\pi_2||_{\infty},\\
	&||q_\kappa^{\pi_1}-q_\kappa^{\pi_2}||_{\infty}\leq L_b ||\pi_1-\pi_2||_{\infty},\\
	\end{align*}
	where $L_a,L_b >0,$ are functions of $\gamma,\kappa,R_{\max}$.
\end{lemma}
\begin{proof}
	We start by proving that $||v^{\pi_1}-v^{\pi_2}||_{\infty}\leq L ||\pi_1-\pi_2||_{\infty}$, i.e, $v^\pi$ is Lipschitz in $\pi$.
	\begin{align}
	||v^{\pi_1}-v^{\pi_2}||_{\infty} &=  
	||T^{\pi_1}v^{\pi_1}-T^{\pi_2}v^{\pi_2}||_{\infty}\nonumber\\
	&\leq  ||T^{\pi_1}v^{\pi_1}-T^{\pi_1}v^{\pi_2} + T^{\pi_1}v^{\pi_2}-T^{\pi_2}v^{\pi_2}||_{\infty} \nonumber \\
	&\leq  ||T^{\pi_1}v^{\pi_1}-T^{\pi_1}v^{\pi_2}||_{\infty} +||T^{\pi_1}v^{\pi_2}-T^{\pi_2}v^{\pi_2}||_{\infty}\nonumber\\
	&\leq \gamma||v^{\pi_1}-v^{\pi_2}||_{\infty} +||T^{\pi_1}v^{\pi_2}-T^{\pi_2}v^{\pi_2}||_{\infty}, \label{eq: online lipchitz main bound}
	\end{align}
	where the last relation is due to the fact $T^{\pi_1}$ is a $\gamma$-contraction. We continue by calculating ${|T^{\pi_1}v^{\pi_2}-T^{\pi_2}v^{\pi_2}|(s)}$.
	\begin{align}
	|T^{\pi_1}v^{\pi_2}-T^{\pi_2}v^{\pi_2}|(s) 
	&\leq |\sum_a \big(\pi_1(a\mid s)-\pi_2(a\mid s)\big)r(s,a)|+\gamma |\sum_{s'}(P^{\pi_1}_{s',s}-P^{\pi_2}_{s',s})v^{\pi_2}(s')|. \label{eq: vpi2 bound 2 terms}
	\end{align}
	
	We bound each term in \eqref{eq: vpi2 bound 2 terms}. The first term can be bounded by,
	\begin{align}
	|\sum_a \big(\pi_1(a\mid s)-\pi_2(a\mid s)\big)r(s,a)| &\leq  \sum_a |\big(\pi_1(a\mid s)-\pi_2(a\mid s)\big)| |r(s,a)| \nonumber \\
	&\leq R_{\mathrm{max}} \max_s  \sum_a |(\pi_1(a\mid s)-\pi_2(a\mid s))|  \nonumber \\
	&= R_{\mathrm{max}} ||\pi_1-\pi_2||_{\infty}. \label{eq: vpi2 bound 2 terms first term}
	\end{align}
	In the first relation we used the triangle inequality and in the second inequality the fact that $|r(s,a)|$ is bounded by $R_{\mathrm{max}}$.

	The second term in \eqref{eq: vpi2 bound 2 terms} can be bounded by,
	\begin{align}
	|\sum_{s'}(P^{\pi_1}_{s',s}-P^{\pi_2}_{s',s})v^{\pi_2}(s')|&=|\sum_{s',a}P(s' \mid s,a)(\pi_1(a\mid s)-\pi_2(a\mid s)) v^{\pi_2}(s')|\nonumber \\
	&\leq \sum_a \sum_{s'} P(s' \mid s,a)|(\pi_1(a\mid s)-\pi_2(a\mid s)) v^{\pi_2}(s')|\nonumber \\
	&\leq \sum_a\sum_{s'} P(s' \mid s,a)|(\pi_1(a\mid s)-\pi_2(a\mid s))| |v^{\pi_2}(s')|\nonumber \\
	&\leq \sum_a\sum_{s'} P(s' \mid s,a)|(\pi_1(a\mid s)-\pi_2(a\mid s))| \frac{R_{\mathrm{max}}}{1-\gamma}\nonumber \\
	&= \sum_a  |(\pi_1(a\mid s)-\pi_2(a\mid s))| \frac{R_{\mathrm{max}}}{1-\gamma} \sum_{s'}  P(s' \mid s,a)\nonumber \\
	&= \sum_a  |(\pi_1(a\mid s)-\pi_2(a\mid s))| \frac{R_{\mathrm{max}}}{1-\gamma} \nonumber \\
	&\leq \max_s \sum_a  |(\pi_1(a\mid s)-\pi_2(a\mid s))| \frac{R_{\mathrm{max}}}{1-\gamma}  =  \frac{R_{\mathrm{max}}}{1-\gamma} ||\pi_1-\pi_2||_{\infty} \label{eq: vpi2 bound 2 terms second term}
	\end{align}
	In the first relation we used the triangle inequality, in the forth relation we used the fact that for any $\pi$ and $s$, $v^{\pi}(s)\leq \frac{R_{\mathrm{max}}}{1-\gamma}$, and in the fifth relation the fact that for any $s$ and $a$, $P(s'\mid s,a)$ is a probability function, thus sums to one.

	Using \eqref{eq: vpi2 bound 2 terms first term}, \eqref{eq: vpi2 bound 2 terms second term} to bound \eqref{eq: vpi2 bound 2 terms} yields that for any $s$,
	\begin{align*}
	|T^{\pi_1}v^{\pi_2}-T^{\pi_2}v^{\pi_2}|(s) \leq \frac{R_{\mathrm{max}}}{1-\gamma}||\pi_1-\pi_2||_{\infty}.
	\end{align*}
	
	Thus, $||T^{\pi_1}v^{\pi_2}-T^{\pi_2}v^{\pi_2}||_{\infty}\leq \frac{R_{\mathrm{max}}}{1-\gamma}||\pi_1-\pi_2||_{\infty}$. Plugging this bound into \eqref{eq: online lipchitz main bound} and rearranging yields,
	\begin{align}
	||v^{\pi_1}-v^{\pi_2}||_{\infty}\leq  \frac{R_{\max}}{(1-\gamma)^2}||\pi_1-\pi_2||_{\infty}, \label{eq:lip_v_pi}
	\end{align}
	giving that $L =  \frac{R_{\max}}{(1-\gamma)^2}$. 
	
	We continue by analysing $||T_\kappa v^{\pi_1}-T_\kappa v^{\pi_2}||_{\infty}$. We remind the reader that $T_\kappa v^{\pi}$ satisfies the following fixed-point equation:
	\begin{align*}
	T_\kappa v^{\pi}(s) &= \max_a\ \left[ r(s,a)+\gamma(1-\kappa)\sum_{s'} P(s'\mid s,a)v^{\pi}(s')+\kappa\gamma \sum_{s'} P(s'\mid s,a)(T_\kappa v^{\pi})(s')\right]\\ 
	&\eqdef \bar{T}^\pi_{\kappa}T_\kappa v^{\pi} (s),
	\end{align*}
	where we defined the `optimal' Bellman operator of the surrogate MDP to be $\bar{T}^\pi_{\kappa}$ (see Remark \ref{rem: applying T kappa}). Furthermore, since this operator is the optimal Bellman operator of a $\kappa\gamma$-discounted MDP, it is a $\kappa\gamma$ contraction mapping.  We now use a similar technique as the above to show ${||T_\kappa v^{\pi_1}-T_\kappa v^{\pi_2}||_{\infty}\leq L_\kappa || \pi_1-\pi_2||_\infty }$, i.e, $T_\kappa v^{\pi}$ is Lipschitz in $\pi$.
	\begin{align*}
	||T_\kappa v^{\pi_1}-T_\kappa v^{\pi_2}||_{\infty} &= ||\bar{T}^{\pi_1}_{\kappa} T_\kappa v^{\pi_1}-\bar{T}^{\pi_2}_{\kappa}T_\kappa v^{\pi_2}||_{\infty}\\
	&\leq  ||\bar{T}^{\pi_1}_{\kappa} T_\kappa v^{\pi_1}-\bar{T}^{\pi_1}_{\kappa}T_\kappa v^{\pi_2}||_{\infty} + ||\bar{T}^{\pi_1}_{\kappa}T_\kappa v^{\pi_2}-\bar{T}^{\pi_2}_{\kappa}T_\kappa v^{\pi_2}||_{\infty}\\
	&\leq  \kappa\gamma||T_\kappa v^{\pi_1}-T_\kappa v^{\pi_2}||_{\infty} + ||\bar{T}^{\pi_1}_{\kappa}T_\kappa v^{\pi_2}-\bar{T}^{\pi_2}_{\kappa}T_\kappa v^{\pi_2}||_{\infty}.
	\end{align*}
	We now bound the second term.
	\begin{align*}
	|\bar{T}^{\pi_1}_{\kappa}T_\kappa v^{\pi_2}-\bar{T}^{\pi_2}_{\kappa}T_\kappa v^{\pi_2}|(s)
	&\leq \max_a \gamma(1-\kappa)| \sum_{s'} P(s'\mid s,a)(v^{\pi_1}-v^{\pi_2})(s')|\\
	&\leq \max_a \gamma(1-\kappa)\sum_{s'} P(s'\mid s,a)|v^{\pi_1}-v^{\pi_2}|(s')\\
	&\leq \max_a \gamma(1-\kappa)\sum_{s'} P(s'\mid s,a)||v^{\pi_1}-v^{\pi_2}||_\infty = \gamma(1-\kappa)||v^{\pi_1}-v^{\pi_2}||_\infty,
	\end{align*}
	where we used the definition of $\bar{T}^{\pi}_{\kappa}$ and the identity $|\max_{x\in \mathcal{X}} f(x)-\max_{x\in \mathcal{X}} g(x)|\leq \max_{x\in \mathcal{X}}|f(x)-g(x)|$ in the first relation and the triangle inequality in the second. 
	
	Using \eqref{eq:lip_v_pi}, we have
	\begin{align*}
	||T_\kappa v^{\pi_1}-T_\kappa v^{\pi_2}||_{\infty}&\leq \frac{\gamma(1-\kappa)}{1-\kappa\gamma}||v^{\pi_1}-v^{\pi_2}||_\infty\\
	&\leq \frac{\gamma(1-\kappa)}{1-\kappa\gamma}\frac{R_{\max}}{(1-\gamma)^2} ||\pi_1-\pi_2||_\infty.
	\end{align*}

	These results transform to results on $q^\pi$ and $q^\pi_\kappa$ as follows. Starting with $q^\pi$,
	\begin{align*}
	|q^{\pi_1}-q^{\pi_2}|(s,a) &= |r(s,a)+\gamma\sum_{s'}P(s'\mid s,a)v^{\pi_1} -r(s,a)-\gamma\sum_{s'}P(s'\mid s,a)v^{\pi_2}|\\ 
	&= \gamma|\sum_{s'}P(s'\mid s,a)(v^{\pi_1} -v^{\pi_2})|\leq \gamma ||v^{\pi_1} -v^{\pi_2}||_\infty.
	\end{align*}
	By taking the max-norm on both sides we get the result since $||v^{\pi_1} -v^{\pi_2}||_\infty$ was shown to be Lipschitz in $\pi$. 
	
	Next, for $q^\pi_\kappa$ we have
	\begin{align*}
	&|q^{\pi_1}_\kappa-q^{\pi_2}_\kappa|(s,a)\\
	=& |\gamma(1-\kappa)\sum_{s'} P(s'\mid s,a)(v^{\pi_1}(s')-v^{\pi_2}(s'))+\kappa\gamma \sum_{s'} P(s'\mid s,a)(T_\kappa v^{\pi_1}-T_\kappa v^{\pi_2})(s')|\\
	\leq &\gamma(1-\kappa)||v^{\pi_1}(s')-v^{\pi_2}(s')||_\infty + \kappa\gamma ||T_\kappa v^{\pi_1}-T_\kappa v^{\pi_2}||_\infty.
	\end{align*}
	By taking the max-norm on both sides we get the result since, as shown above, both $||v^{\pi_1} -v^{\pi_2}||_\infty$ and $||T_\kappa v^{\pi_1}-T_\kappa v^{\pi_2}||_\infty$ are Lipschitz in $\pi$. Finally, since the vector space is finite (due to the finite state and action space), all $L_p$ norms are equivalent. Thus, the Lipschitzness result applies in any $L_p$ norm as well.
\end{proof}

\subsection{Improvement Step}
Here, we prove an equivalent lemma to \cite[Lemma~5.4]{perkins2013asynchronous} which shows that the mean value of the process improves. Denote $b_s \equiv b_s(q^\pi,q^\pi_\kappa,\pi)$ as the policy defined in the Algrorithm \ref{alg:async_kappaPI}. By using Lemma \ref{lemma:kappa_value_difference} and setting $\kappa=0$ we have that
\begin{align*}
v^{(1-\alpha)\pi+\alpha b_s}-v^\pi = \alpha (I-\gamma P^{(1-\alpha)\pi+\alpha b_s})^{-1}(T^{b_s}v^\pi-v^\pi).
\end{align*}
Thus, by taking the limit $\alpha\rightarrow 0$ we have
\begin{align*}
\lim_{\alpha \rightarrow 0} (v^{(1-\alpha)\pi+\alpha b_s}-v^\pi) &= \alpha \nabla_\pi v^{\pi} (b_s-\pi) \\
&= \alpha \  \langle \nabla_\pi v^{\pi}, \Delta \pi \rangle \\
&=\alpha (I-\gamma P^\pi)^{-1}(T^{b_s}v^\pi-v^\pi) +\mathcal{O}(\alpha^2)\geq 0 ,
\end{align*}
where the last inequality is since $T^{b_s}v^\pi-v^\pi\geq 0$ by construction and $(I-\gamma P^\pi)^{-1}\geq 0$ component-wise. We thus get that
\begin{align*}
\frac{1}{\alpha}\lim_{\alpha \rightarrow 0} (v^{(1-\alpha)\pi+\alpha b_s}-v^\pi) = \langle \nabla_\pi v^{\pi}, \Delta \pi \rangle \geq 0.
\end{align*}

\subsection{Convergence of the Algorithm}
We define the same Lyapunov function as defined in \cite[Lemma 5.5]{perkins2013asynchronous}. Due to previous section it is indeed a Lyapunov function since its derivative is negative and the function is bigger than 0 by construction. The presence of the Lyapunov function leads to the convergence of the policy to the optimal policy, similarly to \cite[Corollary~5.6]{perkins2013asynchronous}, which leads to the convergence of $q^\pi$ to  $q^*$. Lastly, since $T_\kappa v^* = v^*$ \cite[Lemma~4]{efroni2018beyond} we have that,
\begin{align*}
q^{\pi^*}_\kappa(\pi') &= r^{\pi'}+\gamma(1-\kappa) P^{\pi'}v^{*} + \kappa\gamma P^{\pi'} T_\kappa v^{*}\\
&= r^{\pi'}+\gamma(1-\kappa) P^{\pi'}v^{*} + \kappa\gamma P^{\pi'} v^{*} \\
&= r^{\pi'}+\gamma P^{\pi'}v^{*}  = q^*(\pi').
\end{align*}
which concludes the proof.

\section{Proof of Lemma \ref{lemma:general_C_improvement}}\label{supp:general_C_improvement}
We first prove a useful lemma that relates the (unnormalized) future distribution, measured in different $\kappa$ scales.
\begin{lemma}\label{lemma:help1}
For any policy $\pi$ and $\kappa,\kappa'\in [0,1],$
\begin{align*}
(I-\xi_{\kappa'} D^{\pi}_{\kappa'}P^{\pi})^{-1} = \frac{\kappa'-\kappa}{1-\kappa} I+\frac{1-\kappa'}{1-\kappa}(I-\xi_\kappa D^{\pi}_{\kappa}P^{\pi})^{-1}.
\end{align*}
\end{lemma}
\begin{proof}
We prove the lemma by using the definition and by some algebraic manipulations.
\begin{align*}
(I-\xi_{\kappa'} D_{\kappa'} ^{\pi}P^{\pi})^{-1} &= (I-\gamma(1-\kappa')(I-\kappa\gamma' P^{\pi})^{-1}P^{\pi})^{-1}\\
&= ((I-\kappa\gamma' P^{\pi})^{-1}(I-\kappa\gamma' P^{\pi}-\gamma(1-\kappa')P^{\pi}))^{-1}\\
&= (I-\gamma P^{\pi})^{-1}(I-\gamma \kappa' P^{\pi})\\
&=(I-\gamma P^{\pi})^{-1}-\kappa' \gamma  P^{\pi}(I-\gamma P^{\pi})^{-1}\\
&=(I-\gamma P^{\pi})^{-1}-\kappa' (I+\gamma  P^{\pi}(I-\gamma P^{\pi})^{-1} - I)\\
&=(I-\gamma P^{\pi})^{-1}-\kappa' ((I-\gamma P^{\pi})^{-1} - I)\\
&=\kappa' I+(1-\kappa')(I-\gamma P^{\pi})^{-1} 
\end{align*}

We see that the following relation holds for any $\kappa\in [0,1]$,
\begin{align*}
(I-\gamma P^{\pi})^{-1} =\frac{1}{1-\kappa}((I-\xi_\kappa D_{\kappa} ^{\pi}P^{\pi})^{-1} - \kappa I).
\end{align*}

Plugging this relation into the previous one we get,
\begin{align*}
(I-\xi_{\kappa'} D_{\kappa'} ^{\pi}P^{\pi})^{-1} &= \kappa' I+(1-\kappa')(I-\gamma P^{\pi})^{-1} \\
&=\kappa' I+\frac{1-\kappa'}{1-\kappa}((I-\xi_{\kappa} D_{\kappa} ^{\pi}P^{\pi})^{-1} - \kappa I)\\
&=\frac{\kappa'-\kappa}{1-\kappa}I+\frac{1-\kappa'}{1-\kappa}(I-\xi_\kappa D_{\kappa} ^{\pi}P^{\pi})^{-1}.
\end{align*}
\end{proof}

We are now ready to prove Lemma \ref{lemma:general_C_improvement}. Assume a constant $C^{\pi^*}_{\kappa}(\mu,\nu)<\infty$ such that,
\begin{align}
d^{\pi^*}_{\kappa,\mu}= (1-\xi)\mu (I-\xi D^{\pi^*}_\kappa)^{-1}<C^{\pi^*}_{\kappa}(\mu,\nu)\nu \label{eq:C_kappa_mu_assumption}.
\end{align}

Given that, we shall calculate $C^{\pi^*}_{\kappa'}(\mu,\nu)$ where $\kappa'>\kappa$. 
\begin{align*}
d^{\pi^*}_{\kappa',\mu}&= (1-\xi_{\kappa'})\mu (I-\xi D^{\pi^*}_{\kappa'})^{-1}\\
&= (1-\xi_{\kappa'})\left(\frac{\kappa'-\kappa}{1-\kappa} \mu+\frac{1-\kappa'}{1-\kappa}\mu((I-\xi_\kappa  D^{\pi}_{\kappa}P^{\pi})^{-1})\right)\\
&\leq  \frac{1-\xi_{\kappa'}}{1-\kappa}\left( (\kappa'-\kappa) \mu+\frac{1-\kappa'}{1-\xi_\kappa}C^{\pi^*}_{\kappa}(\mu,\nu)\nu\right)\\
&=\frac{1-\xi_{\kappa'}}{1-\kappa}(\kappa'-\kappa + \frac{1-\kappa'}{1-\xi_\kappa}C^{\pi^*}_{\kappa}(\mu,\nu))\left(\alpha^*\mu+(1-\alpha^*)\nu\right)\eqdef C^{\pi^*}_{\kappa'}(\mu,\nu(\alpha)) \nu(\alpha),
\end{align*}
where we used Lemma \ref{lemma:help1} in the first line, Equation \ref{eq:C_kappa_mu_assumption} in the second line, and defined $\alpha^*=(1+\frac{1-\kappa'}{(1-\xi_\kappa)(\kappa'-\kappa)}C^{\pi^*}_{\kappa}(\mu,\nu)))^{-1}\in(0,1)$ and $C^{\pi^*}_{\kappa'}(\mu,\nu(\alpha^*))=\frac{1-\xi_{\kappa'}}{1-\kappa}(\kappa'-\kappa + \frac{1-\kappa'}{1-\xi_\kappa}C^{\pi^*}_{\kappa}(\mu,\nu))$. By plugging the expressions of $\xi_{\kappa},\xi_{\kappa'}$ we see that,
\begin{align}
C^{\pi^*}_{\kappa'}(\mu,\nu(\alpha^*))-C^{\pi^*}_{\kappa}(\mu,\nu) &= \frac{1-\xi_{\kappa'}}{1-\kappa}(\kappa'-\kappa + (\frac{1-\kappa'}{1-\xi_\kappa}-\frac{1-\kappa}{1-\xi_{\kappa'}})C^{\pi^*}_{\kappa}(\mu,\nu)) \nonumber\\  
&= \frac{1-\xi_{\kappa'}}{1-\kappa}(\kappa'-\kappa)( 1-C^{\pi^*}_{\kappa}(\mu,\nu)) \label{eq_C_difference}.
\end{align}
Since $C^{\pi^*}_{\kappa}(\mu,\nu)\geq 1$ and $\frac{1-\xi_{\kappa'}}{1-\kappa}(\kappa'-\kappa)>0$ we get that $C^{\pi^*}_{\kappa'}(\mu,\nu(\alpha^*))-C^{\pi^*}_{\kappa}(\mu,\nu)\leq 0$, where the inequality is strict for $C^{\pi^*}_{\kappa}(\mu,\nu)> 1$. Finally, since for $\mu=\nu$ it holds that $\nu(\alpha^*)=(1-\alpha^*)\nu+\alpha^*\nu=\nu$ for, we get that  $C^{\pi^*}_{\kappa}(\nu,\nu)$ is a decreasing function of $\kappa$.

%

\section{Proof of Theorem \ref{theorem:kappa_API}}\label{supp:kappa_API} 
We first prove two technical lemmas.

\begin{lemma}\label{lemma:geometic_series}
	Let $\pi$ be a policy, $\kappa\in[0,1],\ \gamma\in(0,1)$ and $i\in \mathbb{N} \backslash \{0\}$. Then
	\begin{align*}
	(\xi D_\kappa^\pi P^\pi)^i =  \sum_{t={i-1}}^\infty \frac{t!}{(i-1)!(t-(i-1))!}\gamma^{t+1}(1-\kappa)^i\kappa^{t-(i-1)} (P^{\pi})^{t+1},
	\end{align*}
	where, as also given in Definition~\ref{def:kappa_concentrability_coeff}, ${D_\kappa^{\pi}=(1-\kappa\gamma)(I-\kappa\gamma P^\pi)^{-1}.}$
\end{lemma}
\begin{proof}
	First, for any $x\in \mathbb{R}$ s.t $|x|<1$ and $i\in \mathbb{N} \backslash \{0\}$ we have that,
	\begin{align*}
	(1-x)^{-i} = \sum_{t=i-1}^\infty \frac{t!}{(i-1)!(t-(i-1))!} x^{t-(i-1)}.
	\end{align*}
	Since it holds that $||\gamma\kappa P^\pi||=\gamma\kappa<1$, where $||\cdot||$ is the spectral norm of the matrix, we can use the same Taylor expansion when replacing $x$ with $\gamma\kappa P^\pi$. Thus,
	\begin{align}
	(I-\gamma\kappa P^\pi)^{-i} = \sum_{t=i-1}^\infty \frac{t!}{(i-1)!(t-(i-1))!} (\gamma\kappa)^{t-(i-1)} (P^\pi)^{t-(i-1)}. \label{eq:matrix_inv_ith_power}
	\end{align}
	
	Since $D_\kappa^\pi = (1-\kappa\gamma)(I-\kappa\gamma P^\pi)^{-1}$ and any matrix commutes with any function of itself we have that,
	\begin{align*}
	(\xi D_\kappa^\pi P^\pi)^i &= \gamma^i(1-\kappa)^i(D_\kappa^\pi P^\pi)^i =\gamma^i(1-\kappa)^i((I-\kappa\gamma P^{\pi})^{-1})^i (P^\pi)^i.
	\end{align*}
	By using \eqref{eq:matrix_inv_ith_power} and packing the terms we conclude the proof.
	\begin{align*}
	(\xi D_\kappa^\pi P^\pi)^i &= \gamma^i(1-\kappa)^i(I-\kappa\gamma P^{\pi})^{-i} (P^\pi)^i\\
	&= \sum_{t=i-1}^\infty \frac{t!}{(i-1)!(t-(i-1))!}\gamma^{t+1}(1-\kappa)^i\kappa^{t-(i-1)} (P^{\pi})^{t+1}
	\end{align*}
\end{proof}

\begin{lemma}\label{lemma:binomial_trick}
Let $\kappa\in[0,1],\ \gamma\in(0,1), n\in \mathbb{N} \cup \{\infty\}$ and $f:\mathbb{N}\rightarrow \mathbb{R}$. Then
\begin{align*}
&\sum_{l=0}^\infty\sum_{i=1}^{n-1}\sum_{t=i-1}^\infty  \frac{t!}{(i-1)!(t-(i-1))!}\gamma^{t+l+1}\kappa^{t-(i-1)}(1-\kappa)^i  f(t+1+l)\\
\leq &(1-\kappa)\sum_{l=0}^\infty \sum_{t=0}^{n-2}  \gamma^{t+l+1}  f(t+1+l)+ g(\kappa)(1-\kappa)\kappa \sum_{l=0}^\infty\sum_{t=n-1}^{\infty} \gamma^{t+l+1} f(t+1+l),
\end{align*}
where $g(\kappa)$ is a bounded function of $\kappa$. When $n\rightarrow \infty$ the second term vanishes.
\end{lemma}
\begin{proof}
We start by exchanging the summation indices $i$ and $t$. In order to do so, we decouple the summation to two sums. The range of the indices of the first sum is $t\in \{0,..,n-2\}$ and $i\in\{1,..,t+1\}$ and the range of the indices of the second sum is $t\in \{n-1,..,\infty\}$ and $i\in\{1,..,n-1\}$
\begin{align}
&\sum_{l=0}^\infty\sum_{i=1}^{n-1}\sum_{t=i-1}^\infty  \frac{t!}{(i-1)!(t-(i-1))!}\gamma^{t+l+1}\kappa^{t-(i-1)}(1-\kappa)^i  f(t+1+l) \nonumber\\
=&\sum_{l=0}^\infty\sum_{t=0}^{n-2}  \gamma^{t+l+1}  f(t+1+l)\sum_{i=1}^{t+1} \frac{t!}{(i-1)!(t-(i-1))!}\kappa^{t-(i-1)}(1-\kappa)^i \label{eq:binomial_trick_1st} \\
&+\sum_{l=0}^\infty\sum_{t=n-1}^{\infty} \gamma^{t+l+1}  f(t+1+l) \sum_{i=1}^{n-1}  \frac{t!}{(i-1)!(t-(i-1))!}\kappa^{t-(i-1)}(1-\kappa)^i.\label{eq:binomial_trick_2nd}
\end{align}
Let us bound the first sum first \eqref{eq:binomial_trick_1st},
\begin{align*}
&\sum_{l=0}^\infty\sum_{t=0}^{n-2}  \gamma^{t+l+1}  f(t+1+l)\sum_{i=1}^{t+1} \frac{t!}{(i-1)!(t-(i-1))!}\kappa^{t-(i-1)}(1-\kappa)^i\\
=&\sum_{l=0}^\infty\sum_{t=0}^{n-2}  \gamma^{t+l+1}  f(t+1+l)\sum_{i=0}^{t} \frac{t!}{i!(t-i)!}\kappa^{t-i}(1-\kappa)^{i+1}\\
=& (1-\kappa)\sum_{l=0}^\infty\sum_{t=0}^{n-2}  \gamma^{t+l+1}  f(t+1+l),
\end{align*}
where in the first line we changed the index summation $i\gets i-1$ and in the second line we used the binomial identity $\sum_{i=0}^{t} \frac{t!}{i!(t-i)!}\kappa^{t-i}(1-\kappa)^i = (1-\kappa+\kappa)^t=1$.

In order to bound the second term \eqref{eq:binomial_trick_2nd} we define the following function, $\tilde{g}:[n-1,\infty)\rightarrow \mathbb{R}$,
\begin{align*}
\tilde{g}(t)\eqdef \sum_{i=0}^{n-2}  \frac{t!}{i!(t-i)!}\kappa^{t-i}(1-\kappa)^i.
\end{align*}
The function $\tilde{g}(t)$ is a sum of polynomial terms multiplied by a geometric decaying term, $\kappa^t$. Thus, this function is bounded from above, i.e, exists $t^*\in [n-1,\infty)$ such that $\tilde{g}(t)\leq \tilde{g}(t^*),\ \forall t \in[n-1,\infty)$. For such $t^*$, by construction, we have that
\begin{align*}
\sum_{i=1}^{n-1}  \frac{t!}{(i-1)!(t-(i-1))!}\kappa^{t-(i-1)}(1-\kappa)^i&=(1-\kappa)\sum_{i=0}^{n-2}  \frac{t!}{i!(t-i))!}\kappa^{t-i}(1-\kappa)^i\\
&\leq (1-\kappa)\sum_{i=0}^{n-2}  \frac{t^*!}{i!(t^*-i)!}\kappa^{t^*-i}(1-\kappa)^i\\
& =(1-\kappa)\kappa^{t^*-(n-2)}\sum_{i=0}^{n-2}  \frac{t^*!}{i!(t^*-i)!}\kappa^{(n-2)-i}(1-\kappa)^i\\
& \leq(1-\kappa)\kappa \sum_{i=0}^{n-2}  \frac{t^*!}{i!(t^*-i)!}\kappa^{(n-2)-i}(1-\kappa)^i\\
\end{align*}
where the last line holds since for $\kappa\in[0,1],\ t^*\in[n-1,\infty)$ it holds that $\kappa^{t^*-(n-2)}\leq \kappa$. We now define  ${g(\kappa)\eqdef \sum_{i=0}^{n-2}  \frac{t^*!}{i!(t^*-i)!}\kappa^{(n-2)-i}(1-\kappa)^i}$, and observe that it is a bounded function of $\kappa\in [0,1]$, since it is a sum of positive powers of $\kappa$. Thus, \eqref{eq:binomial_trick_2nd} is bounded by
\begin{align*}
&\sum_{l=0}^\infty\sum_{t=n-1}^{\infty} \gamma^{t+l+1}  f(t+1+l) \sum_{i=1}^{n-1}  \frac{t!}{(i-1)!(t-(i-1))!}\kappa^{t-(i-1)}(1-\kappa)^i\\
\leq & g(\kappa)(1-\kappa)\kappa\sum_{l=0}^\infty\sum_{t=n-1}^{\infty} \gamma^{t+l+1}  f(t+1+l) 
\end{align*}

Finally, for the case $n=\infty$ observe we can repeat the same analysis we did for the first term \eqref{eq:binomial_trick_1st} without the need to decouple to two sums. Thus, for this case, the bound on the first term, with $n=\infty$, bounds the expression.

\end{proof}

We are now ready to prove Theorem \ref{theorem:kappa_API}. The proof strategy is similar to the line of work in \cite{farahmand2010error,scherrer2014approximate,lazaric2016analysis}:
Keeping track of the cumulative error and using the definition of $c(i)$ and $c^{\pi^*}(i),$ we bound the performance loss in the $\mu$-weighted $L_1$ norm.

Since the policy in each iteration is an approximate $\kappa$-greedy policy (see Definition \ref{def:approximate_kappa_greedy}), it holds that  $\nu T^{\pi_k}_\kappa v^{\pi_{k-1}}\geq \nu T_\kappa v^{\pi_{k-1}}-\delta$ in each iteration.  Let the error  vector at the $i$-th iteration $\bar{\delta}_i$ satisfy $\nu \bar{\delta}_i \leq \delta$. Thus,
\begin{align}
v^{*}-v^{\pi_k}&=T^{\pi^*}_\kappa v^{*}-T^{\pi^*}_\kappa v^{\pi_{k-1}}+T^{\pi^*}_\kappa v^{\pi_{k-1}}-v^{\pi_k}\nonumber \\
&=\xi D^{\pi^*}_{\kappa} P^{\pi^*}(v^{*}-v^{\pi_{k-1}})+T^{\pi^*}_\kappa v^{\pi_{k-1}}-v^{\pi_k} \nonumber \\
&=\xi D^{\pi^*}_{\kappa} P^{\pi^*}(v^{*}-v^{\pi_{k-1}})+T^{\pi^*}_\kappa v^{\pi_{k-1}}-T^{\pi_k}_\kappa v^{\pi_{k-1}}+T^{\pi_k}_\kappa v^{\pi_{k-1}}-v^{\pi_k} \nonumber \\
&\leq \xi D^{\pi^*}_{\kappa} P^{\pi^*}(v^{*}-v^{\pi_{k-1}})+T^{\pi^*}_\kappa v^{\pi_{k-1}}-\max_{\pi'} T^{\pi'}_\kappa v^{\pi_{k-1}}+\bar{\delta}_i+T^{\pi_k}_\kappa v^{\pi_{k-1}}-v^{\pi_k} \nonumber  \\
&\leq \xi D^{\pi^*}_{\kappa} P^{\pi^*}(v^{*}-v^{\pi_{k-1}})+\bar{\delta}_i+T^{\pi_k}_\kappa v^{\pi_{k-1}}-v^{\pi_k} \nonumber  \\
&= \xi D^{\pi^*}_{\kappa} P^{\pi^*}(v^{*}-v^{\pi_{k-1}})+\bar{\delta}_i+\xi D_\kappa^{\pi_\kappa}P^{\pi_k}(v^{\pi_{k-1}}-v^{\pi_k}), \label{eq: v^* to v^pik relation}
\end{align}
where we used in the second and last relations that for any policy $\pi$, and any value functions $v_1,v_2$, $T_\kappa^\pi v_1 - T_\kappa^\pi v_2 = \xi D_\kappa^\pi P^\pi (v_1-v_2)$. This can be seen by using the definition of $T_\kappa^\pi$ (see Section \ref{sec:dp_multiple_step}). Notice that
\begin{align*}
v^{\pi_{k-1}}-v^{\pi_k}&= T_\kappa^{\pi_{k-1}} v^{\pi_{k-1}}-v^{\pi_k}\\
&\leq \max_{\pi'} T_\kappa^{\pi'} v^{\pi_{k-1}}-v^{\pi_k}\\
&\leq  T_\kappa^{\pi_k} v^{\pi_{k-1}}-v^{\pi_k}+\bar{\delta}_i \\
&=  T_\kappa^{\pi_k} v^{\pi_{k-1}}-T_\kappa^{\pi_k} v^{\pi_k}+\bar{\delta}_i \\
&=\xi D_\kappa^{\pi_k} P^{\pi_k}(v^{\pi_{k-1}}-v^{\pi_k}) +\bar{\delta}_i.
\end{align*}
Hence,
\begin{align}
&(I-\xi D_\kappa^{\pi_k}P^{\pi_k})(v^{\pi_{k-1}}-v^{\pi_k})\leq \bar{\delta}_i, \mbox{ i.e., } \nonumber \\
&v^{\pi_{k-1}}-v^{\pi_k}\leq (I-\xi D_\kappa^{\pi_k}P^{\pi_k})^{-1}\bar{\delta}_i. \label{eq: inverse with delta}
\end{align}
The last equation holds due to \cite[Lemma~4.2]{munos2007performance}, combined with the fact that ${(I-\xi D_\kappa^{\pi_k}P^{\pi_k})^{-1} = \sum_{i=0}^\infty \xi D_\kappa^{\pi_k}P^{\pi_k} \geq 0},$ element-wise.

Plugging \eqref{eq: inverse with delta} into \eqref{eq: v^* to v^pik relation}, we have that
\begin{align*}
v^{*}-v^{\pi_k}&\leq \xi D^{\pi^*}_{\kappa} P^{\pi^*}(v^{*}-v^{\pi_{k-1}}) +\bar{\delta}_i + \xi D_\kappa^{\pi_k}P^{\pi_k}(I-\xi D_\kappa^{\pi_k}P^{\pi_k})^{-1}\bar{\delta}_i\\
&= \xi D^{\pi^*}_{\kappa} P^{\pi^*}(v^{*}-v^{\pi_{k-1}}) +(I-\xi D_\kappa^{\pi_k}P^{\pi_k})^{-1}\bar{\delta}_i, \\
\end{align*}
where the second relation holds since for matrix $X$ s.t. $\|X\| < 1,$ $I + X(I-X)^{-1} = (I-X)^{-1}.$

We thus get that the errors accumulate as follows.
\begin{align*}
v^{*}-v^{\pi_k}&\leq \sum_{i=0}^{k-1} (\xi D_\kappa ^{\pi^*}P^{\pi^*})^i(I-\xi D_\kappa ^{\pi_{k-i}}P^{\pi_{k-i}})^{-1}\bar{\delta}_i +(\xi D_\kappa ^{\pi^*}P^{\pi^*})^k(v^*-v^{\pi_0}).
\end{align*}

We continue by multiplying both sides with $\mu$ and get
\begin{align}
\mu(v^{*}-v^{\pi_k})&\leq \sum_{i=0}^{k-1}  \mu (\xi D_\kappa ^{\pi^*}P^{\pi^*})^i(I-\xi D_\kappa ^{\pi_{k-i}}P^{\pi_{k-i}})^{-1}\bar{\delta}_i + \xi^{k} \frac{R_{\max}}{1-\gamma}. \label{eq:api_accumulated_errors} 
\end{align}

Using Lemma \ref{lemma:help1} with $\kappa=0$ and renaming $\kappa'$ to $\kappa$, we have that
\begin{align*}
(I-\xi D_\kappa ^{\pi_{k-i}}P^{\pi_{k-i}})^{-1} = (1-\kappa)(I-\gamma P^{\pi_{k-i}})^{-1}  + \kappa I. 
\end{align*}

Plugging this relation into \eqref{eq:api_accumulated_errors} gives
\begin{align}
\mu(v^{*}-v^{\pi_k})&\leq \sum_{i=0}^{k-1}  \mu (\xi D_\kappa ^{\pi^*}P^{\pi^*})^i((1-\kappa)(I-\gamma P^{\pi_{k-i}})^{-1}  + \kappa I)\bar{\delta}_i + \xi^{k} \frac{R_{\max}}{1-\gamma} \nonumber\\
&\leq (1-\kappa)\sum_{i=0}^{k-1}  \mu (\xi D_\kappa ^{\pi^*}P^{\pi^*})^i(I-\gamma P^{\pi_{k-i}})^{-1}\bar{\delta}_i+ \kappa \sum_{i=0}^{k-1}   \mu (\xi D_\kappa ^{\pi^*}P^{\pi^*})^i \bar{\delta}_i  + \xi^{k} \frac{R_{\max}}{1-\gamma}. \label{eq: cumulative error}
\end{align}

The following two lemmas provide bounds for the first two terms above. The bounds are composed of the concentrability coefficients (see Definition~\ref{def:concentrability_coeff} and Definition~\ref{def:kappa_concentrability_coeff}).

\begin{lemma}\label{lemma:kappa_api_lemma1}
Let $\kappa\in[0,1]$. For any sequence of policies $\{ \pi_{k-i}\}_{i=0}^{k-1}$, optimal policy $\pi^*,$ and error vector which satisfy $\nu \bar{\delta}_i\leq \delta,$
\begin{equation}
\label{eq:kappa_api_final_bound_1st} 
\sum_{i=0}^{k-1}  \mu (\xi D_\kappa ^{\pi^*}P^{\pi^*})^i(I-\gamma P^{\pi_{k-i}})^{-1} \bar{\delta}_i \leq \left( \frac{(1-\kappa)C^{(2)}(\mu,\nu)}{(1-\gamma)^2}+\frac{\kappa C^{(1)}(\mu,\nu)}{1-\gamma}\right)\delta 
\end{equation}
and
\begin{align}
\label{eq:kappa_api_final_bound_2nd} 
\sum_{i=0}^{k-1}  \mu (\xi D_\kappa ^{\pi^*}P^{\pi^*})^i(I-&\gamma P^{\pi_{k-i}})^{-1} \bar{\delta}_i 
\nonumber\\ &\leq \left(
k \frac{(1-\kappa) C^{(1)}(\mu,\nu)}{1-\gamma} + \frac{\kappa C^{(1)}(\mu,\nu)}{1-\gamma} +\frac{g(\kappa)(1-\kappa)\kappa \gamma^k C^{(2,k)}(\mu,\nu)}{(1-\gamma)^2}\right)\delta.
\end{align}
\end{lemma}

\begin{proof}
We start with proving \eqref{eq:kappa_api_final_bound_1st}. Let $\pi'$ be an arbitrary policy. For $i>k-1,$ we define $\pi_{k-i}=\pi'$ and vectors $ \bar{\delta}_i$ s.t. $\nu\bar{\delta}_i\leq \delta\ $.
\begin{align}
\sum_{i=0}^{k-1}  \mu (\xi D_\kappa ^{\pi^*}P^{\pi^*})^i(I-\gamma P^{\pi_{k-i}})^{-1} \bar{\delta}_i&\leq \sum_{i=0}^{\infty}   \mu (\xi D_\kappa ^{\pi^*}P^{\pi^*})^i(I-\gamma P^{\pi_{k-i}})^{-1}\bar{\delta}_i \nonumber \\
&=\mu (I-\gamma P^{\pi_{k}})^{-1}\bar{\delta}_0 + \sum_{i=1}^{\infty}  \mu (\xi D_\kappa ^{\pi^*}P^{\pi^*})^i(I-\gamma P^{\pi_{k-i}})^{-1}\bar{\delta}_i\label{eq:kappa_api_first_term_1}.
\end{align}

For the first term in \eqref{eq:kappa_api_first_term_1} we have that
\begin{equation}
\mu (I-\gamma P^{\pi_{k}})^{-1}\bar{\delta}_0 = \sum_{l=0}^\infty \gamma^l \mu (P^{\pi_k})^l \bar{\delta}_0
\leq \sum_{l=0}^\infty \gamma^l c(l) \nu \bar{\delta}_0
\leq \sum_{l=0}^\infty \gamma^l c(l) \delta = \frac{C^{(1)}(\mu,\nu)}{1-\gamma}\delta \label{eq:help1_c1_bound},
\end{equation}
where for the second relation we used the definition of the sequence $\{c(i)\}_{i=0}^\infty$ (see Definition \ref{def:concentrability_coeff}) and in the third relation we used $\nu \bar{\delta}_0\leq \delta$ (see Definition \ref{def:approximate_kappa_greedy}).
 
Next, we bound the second term in \eqref{eq:kappa_api_first_term_1}.
\begin{align}
&\sum_{i=1}^{\infty}  \mu (\xi D_\kappa ^{\pi^*}P^{\pi^*})^i(I-\gamma P^{\pi_{k-i}})^{-1}\bar{\delta}_i \nonumber \\
=&\sum_{l=0}^\infty \sum_{i=1}^{\infty} \gamma^l  \mu (\xi D_\kappa ^{\pi^*}P^{\pi^*})^i(P^{\pi_{k-i}})^{l}\bar{\delta}_i \label{eq: jump from} \\ =&\sum_{l=0}^{\infty}\sum_{i=1}^\infty\sum_{t=i-1}^\infty   \frac{t!}{(i-1)!(t-(i-1))!}  \gamma^{l+t+1} \kappa^{t-(i-1)}(1-\kappa)^i \mu(P^{\pi^*})^{t+1} (P^{\pi_{k-i}})^l \bar{\delta}_i \nonumber \\
\leq & \sum_{l=0}^{\infty}\sum_{i=1}^\infty\sum_{t=i-1}^\infty   \frac{t!}{(i-1)!(t-(i-1))!}  \gamma^{l+t+1} \kappa^{t-(i-1)}(1-\kappa)^i c(t+1+l)\delta\nonumber\\
\leq & (1-\kappa)\sum_{l=0}^{\infty} \sum_{t=0}^\infty     \gamma^{l+t+1} c(t+1+l)\delta \label{eq: jump to}\\
= & (1-\kappa)\sum_{l=0}^{\infty} \sum_{t=1}^\infty\gamma^{l+t} c(t+l)\delta\nonumber\\
= & (1-\kappa)\left(\sum_{l=0}^{\infty} \sum_{t=1}^\infty     \gamma^{l+t} c(t+l)+\sum_{l=0}^\infty     \gamma^{l} c(l)-\sum_{l=0}^\infty     \gamma^{l} c(l)\right)\delta\nonumber\\
= & (1-\kappa)\left(\sum_{l=0}^{\infty} \sum_{t=0}^\infty     \gamma^{l+t} c(t+l) - \sum_{l=0}^{\infty}\gamma^{l} c(l)\right)\delta = (1-\kappa)\left(\frac{C^{(2)}(\mu,\nu)}{(1-\gamma)^2}-\frac{C^{(1)}(\mu,\nu)}{1-\gamma}\right)\delta \label{eq:help2_c1_bound}.
\end{align}
For the first relation we used the Taylor expansion $(I-\gamma P^{\pi_{k-i}})^{-1} = \sum_{l=0}^\infty \gamma^l \left(P^{\pi_{k-i}}\right)^l,$ for the second we used Lemma~\ref{lemma:geometic_series}, for the third we used the definition of the sequence $\{c(i)\}_{i=0}^\infty$ and $\nu \bar{\delta}_i \leq \delta,$ for the fourth we applied Lemma \ref{lemma:binomial_trick} with $n=\infty$ and $f(\cdot)=c(\cdot),$ and for the fifth we shifted the summation index $t\gets t+1$.

We bound \eqref{eq:kappa_api_first_term_1} by summing the bounds in \eqref{eq:help1_c1_bound} and \eqref{eq:help2_c1_bound} to obtain the first statement of the lemma, \eqref{eq:kappa_api_final_bound_1st}.

To prove the second statement, \eqref{eq:kappa_api_final_bound_2nd}, we again split expression of interest, similarly to \eqref{eq:kappa_api_first_term_1}.
\begin{align}
\sum_{i=0}^{k-1}  \mu (\xi D_\kappa ^{\pi^*}P^{\pi^*})^i(I-\gamma P^{\pi_{k-i}})^{-1} \bar{\delta}_i\leq \mu (I-\gamma P^{\pi_{k}})^{-1}\bar{\delta}_0 + \sum_{i=1}^{k-1}  \mu (\xi D_\kappa ^{\pi^*}P^{\pi^*})^i(I-\gamma P^{\pi_{k-i}})^{-1}\bar{\delta}_i \label{eq:kappa_api_first_term_2}.
\end{align}
As in \eqref{eq:help1_c1_bound}, the first term in \eqref{eq:kappa_api_first_term_2} is bounded by
\begin{align}
\mu (I-\gamma P^{\pi_{k}})^{-1}\bar{\delta}_0\leq \frac{C^{(1)}(\mu,\nu)}{1-\gamma}\delta. \label{eq:help1_c1_bound_similar}
\end{align}
Next, we bound the second term in \eqref{eq:kappa_api_first_term_2}.
\begin{align}
&\sum_{i=1}^{k-1}  \mu (\xi D_\kappa ^{\pi^*}P^{\pi^*})^i(I-\gamma P^{\pi_{k-i}})^{-1}\bar{\delta}_i \nonumber\\
= &\sum_{l=0}^\infty\sum_{i=1}^{k-1}  \gamma^l \mu (\xi D_\kappa ^{\pi^*}P^{\pi^*})^i(P^{\pi_{k-i}})^l\bar{\delta}_i \nonumber\\
\leq & (1-\kappa)\sum_{l=0}^\infty \sum_{t=0}^{k-2}  \gamma^{t+1+l}  c(t+1+l)\delta+g(\kappa)(1-\kappa)\kappa  \sum_{l=0}^\infty\sum_{t=k-1}^{\infty} \gamma^{t+1+l} c(t+1+l)\delta \nonumber\\
= & (1-\kappa) \sum_{t=0}^{k-2}\sum_{l=0}^\infty  \gamma^{t+1+l}  c(t+1+l)\delta+g(\kappa)(1-\kappa)\kappa \gamma^k \sum_{l=0}^\infty\sum_{t=0}^{\infty} \gamma^{t+l} c(t+l+k)\delta \nonumber\\
\leq &(k-1) \frac{(1-\kappa) C^{(1)}(\mu,\nu)}{1-\gamma} \delta+\frac{g(\kappa)(1-\kappa)\kappa \gamma^k C^{(2,k)}(\mu,\nu)}{(1-\gamma)^2}\delta. \label{eq:help3_c1_bound}
\end{align}
In the first relation we used the Taylor expansion of $(I-\gamma P^{\pi_{k-i}}).$ For the second relation we perform the same steps as from \eqref{eq: jump from} to \eqref{eq: jump to}, where this time we used Lemma \ref{lemma:binomial_trick} with finite $n=k.$

Summing the terms in \eqref{eq:help1_c1_bound_similar} and \eqref{eq:help3_c1_bound}, we obtain the second statement of the lemma, \eqref{eq:kappa_api_final_bound_2nd}.
\end{proof}

%
%

\begin{lemma}\label{lemma:kappa_api_lemma2}
Let $\kappa\in[0,1]$. For any sequence of policies $\{ \pi_{k-i}\}_{i=0}^{k-1}$, optimal policy $\pi^*,$ and error vectors which satisfy $\nu \bar{\delta}_i\leq \delta,,$
\begin{equation}
\label{eq:second_term_kappa_api_final_bound_1st} 
\sum_{i=0}^{k-1}  \mu (\xi D_\kappa ^{\pi^*}P^{\pi^*})^i \bar{\delta}_i \leq \frac{1-\kappa\gamma}{1-\gamma}C^{\pi^*(1)}_\kappa(\mu,\nu) \delta
\end{equation}
and
\begin{equation}
\label{eq:second_term_kappa_api_final_bound_2nd} 
\sum_{i=0}^{k-1}  \mu (\xi D_\kappa ^{\pi^*}P^{\pi^*})^i \bar{\delta}_i\leq  k \frac{1- \kappa\gamma}{1-\gamma} C^{\pi^*}_\kappa(\mu,\nu) \delta.
\end{equation}
\end{lemma}

\begin{proof} We begin proving the first statement. For $i>k-1,$ we define vectors $\bar{\delta}_i$ s.t. $\nu \bar{\delta}_i \leq \delta.$ Thus,
\begin{align}
\sum_{i=0}^{k-1}  \mu (\xi D_\kappa ^{\pi^*}P^{\pi^*})^i \bar{\delta}_i &\leq  \mu\bar{\delta}_0+\sum_{i=1}^{\infty}  \mu (\xi D_\kappa ^{\pi^*}P^{\pi^*})^i \bar{\delta}_i. \label{eq:kappa_api_second_term1}
\end{align}
For the first term in \eqref{eq:kappa_api_second_term1},
\begin{align}
\mu\bar{\delta}_0\leq c(0)\nu\bar{\delta}_0 \leq c(0) \delta, \label{eq:help4_c1_bound}
\end{align}
where we used Definition~\ref{def:concentrability_coeff} and then Definition~\ref{def:approximate_kappa_greedy}.

For the second term in \eqref{eq:kappa_api_second_term1}, we have
\begin{align}
&\sum_{i=1}^{\infty}  \mu (\xi D_\kappa ^{\pi^*}P^{\pi^*})^i \bar{\delta}_i \nonumber\\
=&\sum_{i=1}^{\infty} \sum_{t={i-1}}^\infty   \frac{t!}{(i-1)!(t-(i-1))!}\gamma^{t+1} (1-\kappa)^i\kappa^{t-(i-1)} \mu (P^{\pi^*})^{t+1}  \bar{\delta}_i\nonumber\\
\leq &\sum_{i=1}^{\infty} \sum_{t={i-1}}^\infty   \frac{t!}{(i-1)!(t-(i-1))!}\gamma^{t+1} (1-\kappa)^i\kappa^{t-(i-1)} c^{\pi^*}(t+1)\delta\nonumber\\
\leq &(1-\kappa)\sum_{t=0}^{\infty}  \gamma^{t+1}  c^{\pi*}(t+1) \delta\nonumber\\
= &(1-\kappa)\sum_{t=0}^{\infty}  \gamma^{t}  c^{\pi*}(t)\delta -(1-\kappa)c(0)\delta= \frac{(1-\kappa)C^{\pi^*(1)}(\mu,\nu)}{1-\gamma}\delta-(1-\kappa)c(0)\delta \label{eq:help5_c1_bound}.
\end{align}
For the first relation we apply Lemma~\ref{lemma:geometic_series}, for the second we use the definition of $\{c^{\pi^*}(i)\}_{i=0}^\infty$ and use $\nu \bar{\delta}_i\leq \delta$. For the third relation we apply Lemma~\ref{lemma:binomial_trick} with $n=\infty$, $f(\cdot)=c^{\pi^*}(\cdot)$ and drop the $l$ summation.

Summing the terms in \eqref{eq:help4_c1_bound} and \eqref{eq:help5_c1_bound}, we get
\begin{align*}
\sum_{i=0}^{k-1}  \mu (\xi D_\kappa ^{\pi^*}P^{\pi^*})^i \bar{\delta}_i &\leq \frac{1}{1-\gamma}\left((1-\kappa)C^{\pi^*(1)}(\mu,\nu)+(1-\gamma)\kappa c(0)\right)\delta = \frac{1-\kappa\gamma}{1-\gamma}C^{\pi^*(1)}_\kappa(\mu,\nu)\delta,
\end{align*}
where we identify $C^{\pi^*(1)}_\kappa(\mu,\nu)$ to be the same expression as in Definition~\ref{def:kappa_concentrability_coeff}.

For the second statement of the lemma, \eqref{eq:second_term_kappa_api_final_bound_2nd}, we use the identity $(\xi D_\kappa^{\pi^*}P^{\pi^*})^i\leq (I-\xi D_\kappa^{\pi^*}P^{\pi^*})^{-1}$:
\begin{align*}
\sum_{i=0}^{k-1}  \mu (\xi D_\kappa ^{\pi^*}P^{\pi^*})^i \bar{\delta}_i &\leq \sum_{i=0}^{k-1}  \mu (I-\xi D_\kappa^{\pi^*}P^{\pi^*})^{-1} \bar{\delta}_i\\
&\leq \sum_{i=0}^{k-1}  \frac{C^{\pi^*}_\kappa(\mu,\nu)}{1-\xi}\nu \bar{\delta}_i \leq k \frac{C^{\pi^*}_\kappa(\mu,\nu)}{1-\xi} \delta = k\frac{1-\kappa\gamma}{1-\gamma}C^{\pi^*}_\kappa(\mu,\nu) \delta,
\end{align*}
where the second relation holds due to the definition of $C^{\pi^*}_\kappa(\mu,\nu)$.
\end{proof}

So far, the proof went as follows. First, we expressed the cumulative error in \eqref{eq: cumulative error} as the sum of three terms. Bounding the first and second terms is done with Lemmas~\ref{lemma:kappa_api_lemma1} and \ref{lemma:kappa_api_lemma2}, respectively. Each of those two lemmas gives bounds of two forms.  These two forms correspond to the two statements in Theorem~\ref{theorem:kappa_API}. We now apply the bounds so as to obtain the first statement. Specifically, plugging \eqref{eq:kappa_api_final_bound_1st}  and \eqref{eq:second_term_kappa_api_final_bound_1st} into \eqref{eq: cumulative error} gives the first statement in Theorem~\ref{theorem:kappa_API}.

To obtain the second statement of Theorem~\ref{theorem:kappa_API}, we apply the second form of the bounds in Lemmas~\ref{lemma:kappa_api_lemma1} and \ref{lemma:kappa_api_lemma2}. Specifically, we plug \eqref{eq:kappa_api_final_bound_2nd}  and  \eqref{eq:second_term_kappa_api_final_bound_2nd}  into \eqref{eq: cumulative error}. This gives
\begin{align*}
&\mu(v^{*}-v^{\pi_k})\\
\leq &\left(k \frac{\kappa C^{\pi^*}_\kappa(\mu,\nu)}{1-\xi} + k \frac{(1-\kappa)^2 C^{(1)}(\mu,\nu)}{1-\gamma}+\frac{(1-\kappa)\kappa C^{(1)}(\mu,\nu)}{1-\gamma}+\frac{g(\kappa)(1-\kappa)^2\kappa \gamma^{k} C^{(2,k)}(\mu,\nu)}{(1-\gamma)^2} \right)\delta  \\ &+\xi^{k} \frac{R_{\max}}{1-\gamma}.
\end{align*}
We now carefully choose the iteration number $k$ to make the last term smaller than $\delta$:
\begin{align}
k^*=\left\lceil  \frac{\log{\frac{R_{max}}{\delta(1-\gamma)}}}{1-\xi} \right\rceil=\left\lceil  \frac{(1-\kappa\gamma)\log{\frac{R_{max}}{\delta(1-\gamma)}}}{1-\gamma} \right\rceil. \label{eq:set_iteration_num}
\end{align}
By doing so we see that $\xi^{k
	^*}\frac{R_{\max}}{1-\gamma} < \delta$ and obtain the second statement of the result.

\section{Proof of Theorem \ref{theorem:kappa_PSDP}}\label{supp:kappa_PSDP} 
Here, we merely follow the arguments of \cite[Appendix~A]{scherrer2014approximate}, while using the operators $T_\kappa^\pi$ instead of $T^\pi$ and the approximate operator defined in Definition \ref{def:approximate_kappa_greedy}. As in Section~\ref{supp:kappa_API}, we define the component-wise error at the $i$-th iteration, $\bar{\delta}_i$, which satisfies $\nu \bar{\delta}_i \leq \delta$. We have that for all $k$,
  \begin{align*}
    v^*-v^{ \sigma_{\kappa,k}} & = T_\kappa^{\pi^*} v^* - T_\kappa^{\pi^*}v^{\sigma_{k-1}} + T_\kappa^{\pi^*}v^{\sigma_{k-1}}  - T_\kappa^{\pi_k}v^{\sigma_{k-1}}  \\
    & \le \xi D^{\pi^*}_\kappa P^{\pi^*}(v^* - v^{\sigma_{k-1}} ) + \bar{\delta}_k.
  \end{align*}
  Thus, by induction on $k$, we obtain:
  \begin{align*}
    v^*-v^{ \sigma_{\kappa,k}} &\le \sum_{i=0}^{k-1} (\xi D_\kappa^{\pi^*}P^{\pi^*})^i \bar{\delta}_i + (\xi D_\kappa^{\pi^*}P^{\pi^*})^k (v^*-v^{\pi_0})\\
    &\leq  \sum_{i=0}^{k-1} (\xi D_\kappa^{\pi^*}P^{\pi^*})^i \bar{\delta}_i+ \xi^k \frac{R_{\max}}{1-\gamma}
  \end{align*}
  
We can directly bound this term by applying Lemma \ref{lemma:kappa_api_lemma2}. The two statements in that lemma lead to the two statements in Theorem~\ref{theorem:kappa_PSDP}. Again, for the second statement, we set $k$ as in \eqref{eq:set_iteration_num}. 

\end{document}